\documentclass[letterpaper, 10 pt, conference]{ieeeconf}  

\IEEEoverridecommandlockouts                              

\usepackage[usenames,dvipsnames,table,xcdraw]{xcolor}  
\usepackage{hyperref}
\usepackage{amsmath,amsfonts}
\usepackage{algorithm}                    
\usepackage{algorithmicx}                  
\usepackage[noend]{algpseudocode}
\usepackage{graphicx}      
\usepackage{multirow}      
\usepackage{subcaption}
\usepackage{svg}

\usepackage{amsthm}

\newtheoremstyle{nameddefinition}
  {3pt}   
  {3pt}   
  {\normalfont} 
  {}      
  {\bfseries} 
  {.}     
  { }     
  {\thmname{#1}\thmnumber{ #2}\thmnote{: #3}}

\theoremstyle{nameddefinition}
\newtheorem{definition}{Definition}
\newtheorem{observation}{Observation}
\newtheorem{proposition}{Proposition}
\newtheorem{assumption}{Assumption}
\newtheorem{corollary}{Corollary}
\newtheorem{lemma}{Lemma}
\newtheorem{theorem}{Theorem}

\newcommand{\cspace}{\ensuremath{\mathcal{C}_{\mathrm{space}}}}
\newcommand{\cspacei}{\ensuremath{\mathcal{C}_{i}}}
\newcommand{\cfreei}{\ensuremath{\mathcal{C}_{i}^{\mathrm{free}}}}

\newcommand{\compcspace}{\ensuremath{\mathcal{C}_{\mathrm{comp}}}}
\newcommand{\compfree}{\ensuremath{\mathcal{C}_{\mathrm{comp}}^{\mathrm{free}}}}

\newcommand{\sbmp}{{\sc SBMP}}

\newcommand{\mrmp}{{\sc MRMP}}

\newcommand{\arc}{{\sc ARC}}
\newcommand{\barc}{{\sc Bounded ARC}}
\newcommand{\aoarc}{{\sc AO-ARC}}
\newcommand{\aorrtc}{{\sc AORRTC}}
\newcommand{\rrtc}{{\sc RRTC}}
\newcommand{\rrtstar}{{\sc RRT*}}
\newcommand{\drrtstar}{{d\sc{RRT*}}}
\newcommand{\strrtstar}{{\sc ST-RRT*}}
\newcommand{\ppstrrtstar}{{\sc PP-ST-RRT*}}

\newcommand{\soc}{sum-of-cost}
\newcommand{\makespan}{makespan}
\newcommand{\asao}{a.s.a.o.}
\newcommand{\aox}{{\sc AO-$x$}}

\title{\LARGE \bf
AO-ARC: Almost-Surely Asymptotically Optimal Multi-Robot Motion Planning with ARC
}

\author{
James D. Motes$^{1}$,  Marco Morales$^{1,2}$, and Nancy M. Amato$^{1}$
\thanks{$^{1}$James D. Motes, Marco Morales, and Nancy M. Amato are with the Parasol Lab, School of Computing and Data Science, University of Illinois at Urbana Champaign, Champaign, IL, 61820 USA.
\{\tt jmotes2, moralesa, namato\}@illinois.edu}%
\thanks{$^{2}$Marco Morales is also with the Department of Computer Science at Instituto Tecnol\'ogico Aut\'onomo de M\'exico (ITAM), Mexico City, M\'exico.}
}

\begin{document}

\maketitle
\thispagestyle{empty}
\pagestyle{empty}

\begin{abstract}
We present {\aoarc}, an anytime multi-robot motion planning ({\mrmp}) method that achieves initial solution times on par with state-of-the-art {\mrmp} feasibility solvers while converging faster and more reliably than existing anytime {\mrmp} methods as the number of robots increases.
{\aoarc} adapts the {\aox} meta-algorithm for converting feasibility solvers into anytime algorithms by iteratively calling the original {\arc} method on bounded {\mrmp} instances under a makespan cost metric.
This exploits the adaptive (de)coupling of {\arc} while maintaining the consistent cost bound across robot (de)compositions needed for {\aox}.
We provide theoretical analysis proving the asymptotic optimality properties of {\aoarc} and conduct empirical evaluation on a set of 2D scenarios with different levels of coordination complexity and a 3D manipulator scenario representative of real-world applications.
\end{abstract}

\vspace{-1em}
\section{Introduction}
\label{sec:introduction}

Multi-robot motion planning ({\mrmp}) seeks to \textit{quickly} find \textit{high quality} paths transitioning a team of robots between their respective start and goal locations.
Existing {\mrmp} methods have historically made a choice between quickly finding solutions \textit{or} finding high quality solutions.
This choice was forced by the exponential expansion of the composite state space as the number of robots increases.
Decoupled methods, typically using a prioritized approach, plan for robots individually to avoid the composite space, sacrificing coordination and guarantees for faster planning times.
Coupled approaches accept the cost of directly searching the composite space in order to achieve higher levels of coordination and provide completeness and solution quality guarantees.
The Adaptive Robot Coordination (\arc) approach seeks to balance these by planning individual robot paths first, and then introducing coupled subproblems around detected conflicts in the individual paths, focusing coordination only where it is needed. 
Although {\arc} can maintain probabilistic completeness with these local subproblems, it does not provide asymptotic solution-quality guarantees.

This paper studies how to modify {\arc} to provide both quick planning \textit{and} high quality solution.
We consider total team solution time, a {\makespan}-style metric that measures the time until the last robot reaches its goal, rather than a {\soc}-style metric.
This metric induces a single global deadline that applies uniformly across {\arc}, including the individual robots, local conflict resolution subproblems, and the full composite problem.

\begin{figure}[t]
    \centering
    \hspace{-1.2em}
    \begin{subfigure}[t]{0.165\textwidth}
        \centering
        \includegraphics[width=\linewidth]{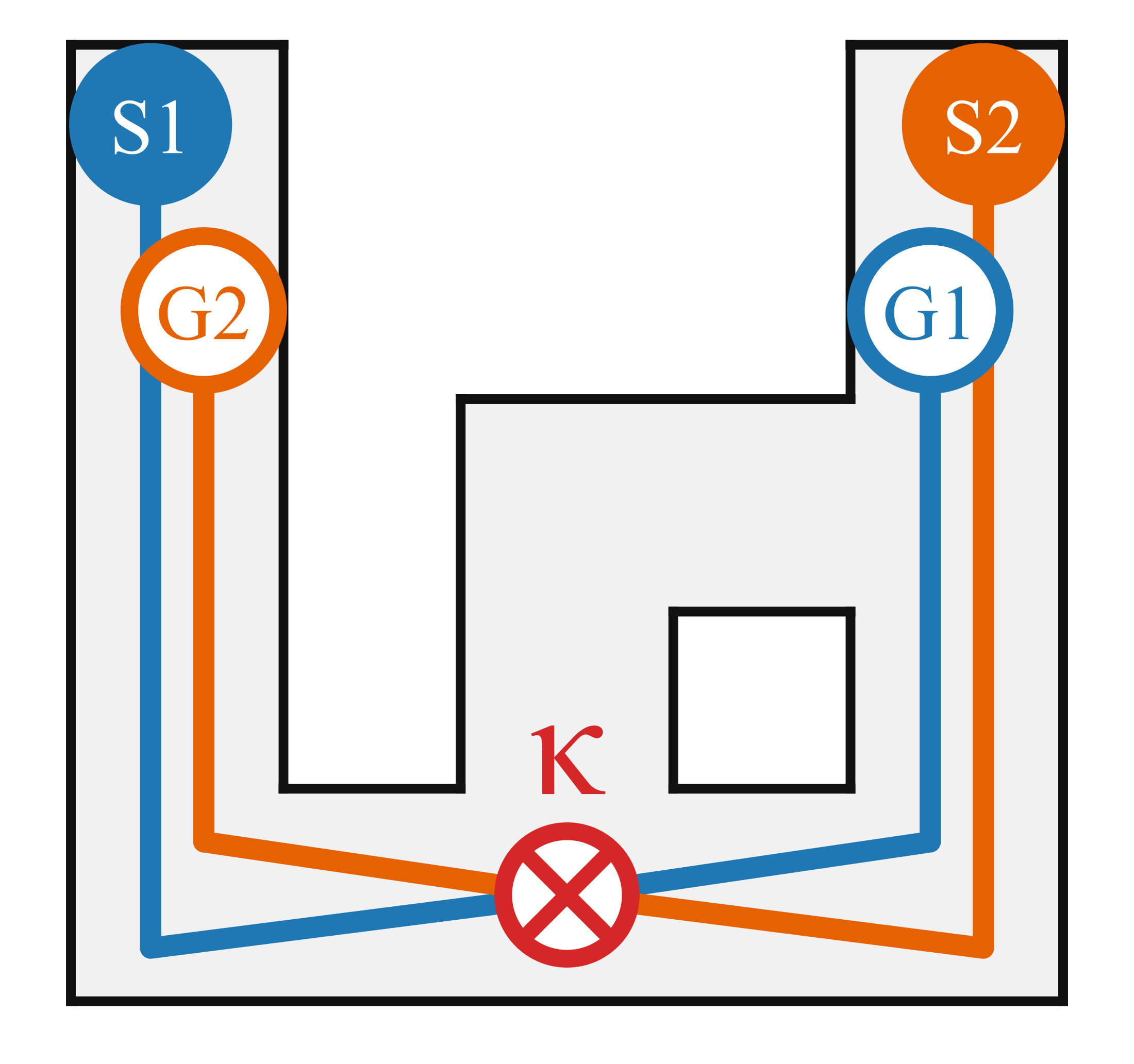}
        \caption{Initial Paths}
        \label{fig:repair-initial}
    \end{subfigure}
    \hspace{-0.6em}
    \begin{subfigure}[t]{0.165\textwidth}
        \centering
        \includegraphics[width=1.0\linewidth]{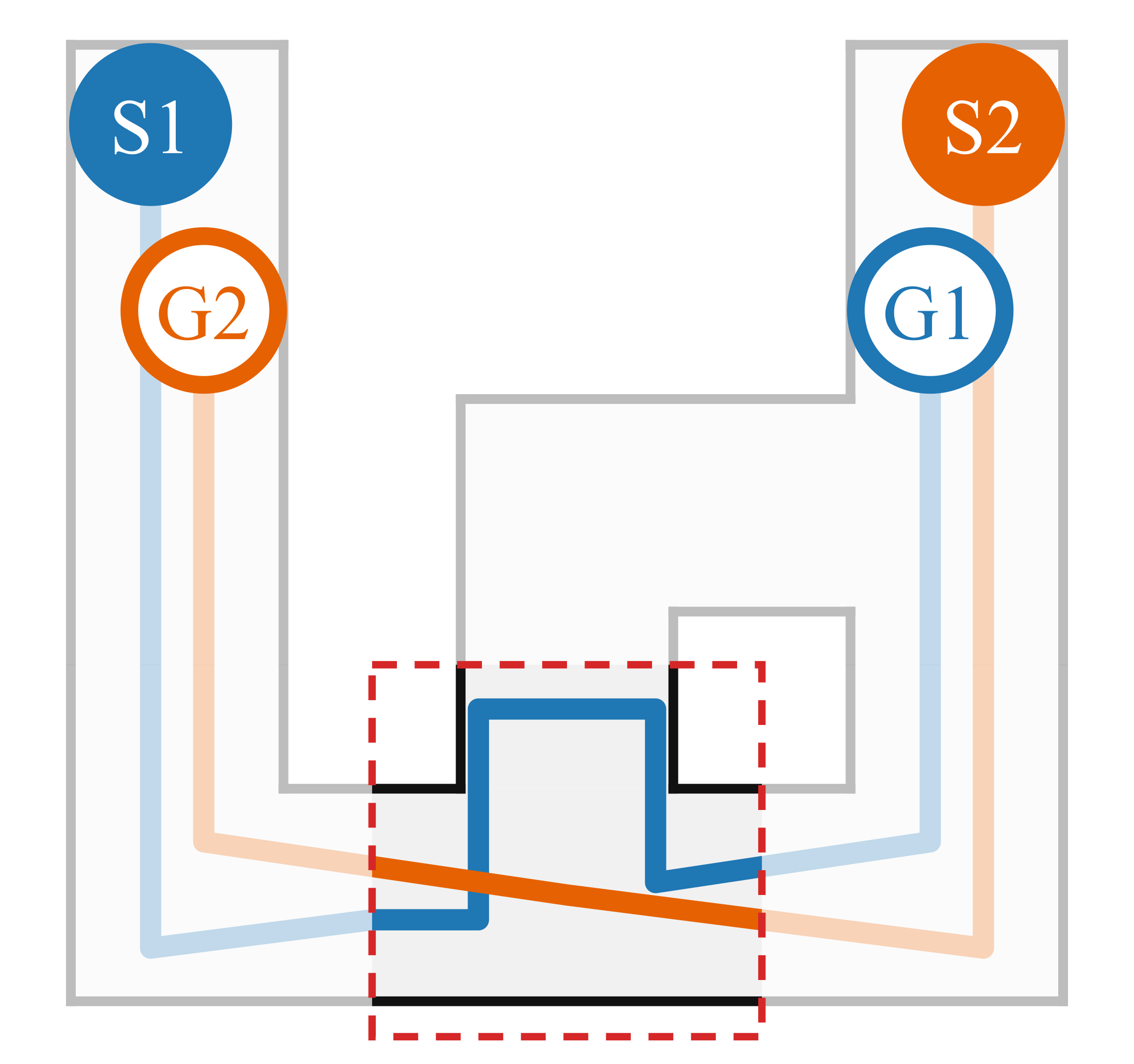}
        \caption{Local Region}
        \label{fig:repair-local}
    \end{subfigure}
    \hspace{-0.6em}
    \begin{subfigure}[t]{0.16\textwidth}
        \centering
        \includegraphics[width=1.0\linewidth]{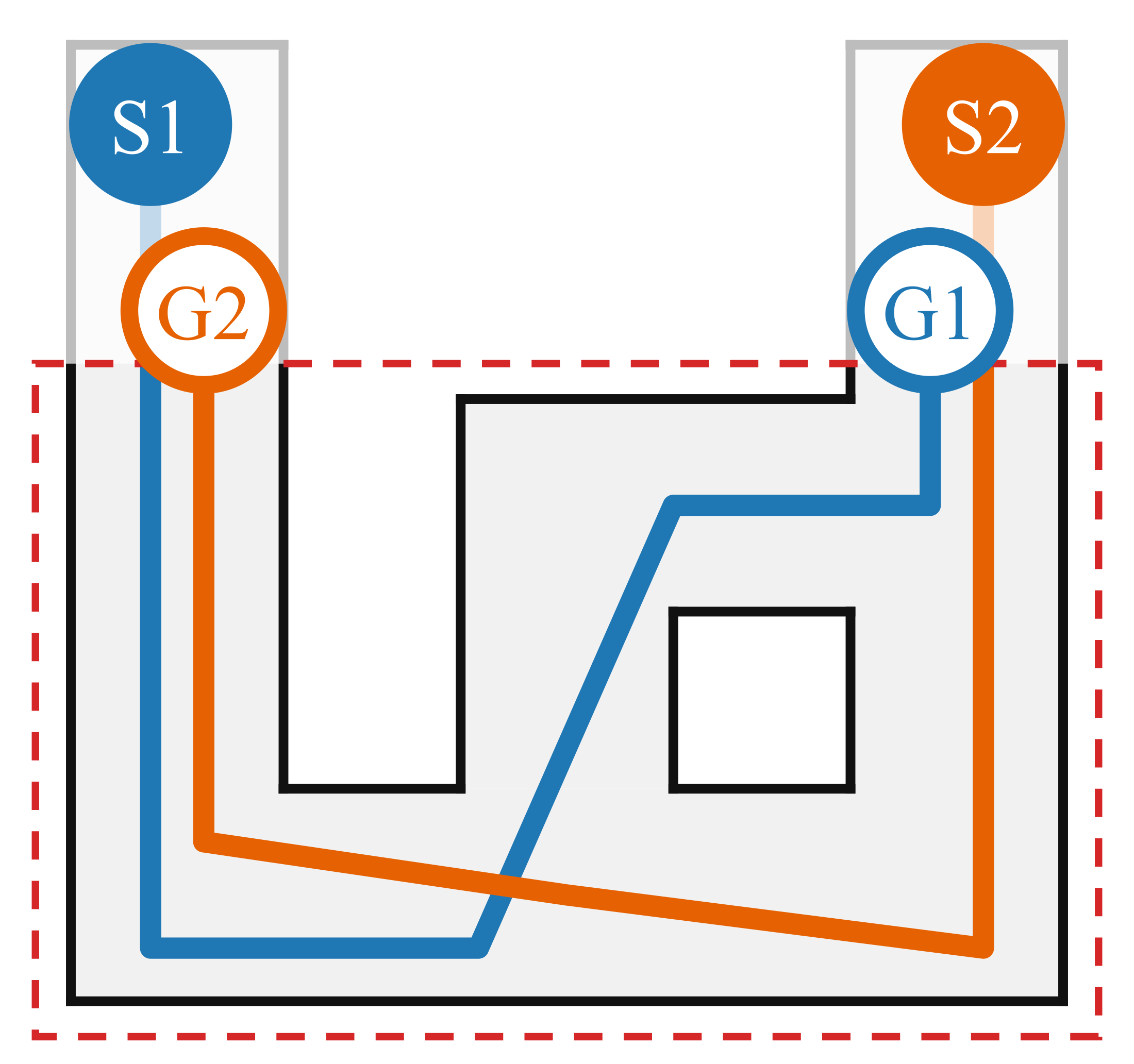}
        \hspace{-0.6em}
        \caption{Expanded Region}
        \label{fig:repair-expanded}
    \end{subfigure}
    \hspace{-1.2em}
    \caption{
    \small
    {\arc} and {\aoarc} both start by planning an initial path for each robot independently.
    Each conflict $\kappa$ in these paths is resolved by creating a subproblem in a local region around the conflict.
    The solution to the subproblem is used to patch the conflicting paths.
    {\arc}, as a feasibility planner, stops here and thus cannot offer solution quality guarantees because of the local decision making.
    {\aoarc} uses the {\aox} meta-algorithm's state-cost formulation to force subproblem expansion until it includes the optimal solution in the local region.
    }
    \label{fig:repair}
    \vspace{-1em}
\end{figure}

We propose {\aoarc}, an almost-surely asymptotically optimal (\asao) extension of {\arc} based on the {\aox} meta-algorithm~\cite{hauser2016asymptotically}.
{\aoarc} iteratively solves bounded {\mrmp} feasibility problems under a team-time bound $B$ and tightens $B$ whenever a better solution is found.
Within each bounded attempt, a bounded {\rrtc}~\cite{wilson2025aorrtc} is used to compute individual robot paths and a composite variant is used to solve local subproblems.
Individual robot paths are given the global bound $B$, and a local bound is computed for subproblems such that the conflict resolution cannot result in global paths violating $B$.
When no bounded solution is found in the current restricted local subproblem, the subproblem is expanded; in the worst case, repeated expansion and conflict detection recover the full composite bounded problem.
The resulting algorithm preserves {\arc}'s adaptive use of coupling while using the team-time bound to connect local repairs to global solution quality.

We evaluate {\aoarc} on a set of 2D mobile robots and planar manipulator scenarios to evaluate the relationship of planning time and solution quality with the number of robots and induced conflicts.
Additionally, we demonstrate improved convergence compared to state-of-the-art {\sc AO-MRMP} solvers on a set of randomly generated tasks for teams of 4 and 8 Panda manipulators in a shared workspace.
In summary, our contributions are:
\begin{itemize}
    \item A makespan-based {\aox} formulation for MRMP
    that remains compatible with {\arc}'s adaptive robot
    decompositions.
    \item {\aoarc}, an anytime {\arc} variant that combines
    bounded feasibility with local conflict repair.

    \item An almost-sure asymptotic optimality analysis for
    {\aoarc} under the stated assumptions.

    \item An empirical evaluation showing fast initial planning
    and improved anytime performance across mobile-robot
    and manipulator benchmarks.
\end{itemize}

\section{Background and Related Work}
\label{sec:related-work}

We review the planning ideas most relevant to {\aoarc}:
asymptotically optimal planning via bounded feasibility, {\mrmp}
coupling strategies, and {\arc}~\cite{solis2024adaptive}.

\subsection{Asymptotically Optimal Sampling-Based Planning}
\label{rw:ao-sbmp}


Sampling-based motion planning ({\sbmp}) has been extended from feasibility to asymptotic optimality through methods such as rewiring-based tree search~\cite{kf-sbaomp-ijrr-11}, informed sampling~\cite{gammell2020batch}, and optimized collision checking~\cite{wilson2025nearest}.
{\aox} instead converts a cost-bounded planning problem into a feasibility problem in an augmented state-cost space and repeatedly calls a feasibility planner under a tightening cost bound~\cite{hauser2016asymptotically}.
If the feasibility planner is well behaved, meaning that it solves feasible bounded queries and contracts expected suboptimality, the resulting anytime planner is asymptotically optimal.
{\aorrtc} recently instantiated {\aox} with {\rrtc}-style feasible planning~\cite{wilson2025aorrtc}.
{\aoarc} applies this bounded-feasibility view to {\mrmp} with a makespan objective.


\subsection{Multi-Robot Motion Planning}
\label{rw:mrmp}

{\mrmp} seeks collision-free trajectories for multiple robots from assigned starts to goals.
The main difficulty is that the composite configuration space {\compcspace} grows exponentially with the number of robots.
Coupled methods search this composite space directly or implicitly, enabling strong coordination and, in some cases, completeness or optimality guarantees~\cite{sl-uppccdpmrs-2002,ssdhb-dsaiammp-20}.
Decoupled methods, including prioritized planners, plan for robots sequentially and treat higher priority robots as dynamic obstacles~\cite{van2005prioritized,kerimov2025si}.
These methods are often fast but generally lose team-level completeness or optimality guarantees.

Hybrid methods attempt to obtain the speed of decoupled planning while escalating to coupled planning only when needed~\cite{wc-sefmpp-15,smsa-rmmpucs-21,solis2024adaptive}.
{\arc} is a hybrid approach that begins with independent robot paths, detects conflicts, and creates coupled local subproblems around those conflicts.
Subproblems are expanded in time, $\cspace$ bounds, and robot set when local repairs fail. 
In the limit, this expansion can recover the full composite problem.
With a probabilistically complete subproblem solver, {\arc} is probabilistically complete, but its local repair decisions do not provide solution quality guarantees.
{\aoarc} preserves this adaptive coupling mechanism while using {\aox} bounded feasibility to obtain anytime improvement under a makespan bound.

\subsubsection{Asymptotically Optimal MRMP}
\label{rw:ao-mrmp}

Compared with single robot planning, relatively little work has studied asymptotically optimal sampling-based {\mrmp}.
Direct application of single robot asymptotically optimal methods to the composite space ignores exploitable multi-robot structure~\cite{hartmann2025sampling}.
{\drrtstar} builds individual roadmaps and searches an implicit tensor-product graph with {\rrtstar}-style rewiring, yielding asymptotic optimality for several metrics~\cite{ssdhb-dsaiammp-20}.
{\strrtstar} is asymptotically optimal for single robot space-time planning~\cite{grothe2022st}.
When used inside a prioritized multi-robot planner~\cite{kerimov2025si}, however, it does not provide team-level optimality.
We compare against these coupled, decoupled, and composite space baselines in Section~\ref{sec:evaluation}.

\section{Problem Definition}
\label{sec:problem}

This section formulates the {\mrmp} problem in the form required by {\aox}~\cite{hauser2016asymptotically}.
We first define the bounded-speed {\mrmp} problem.
We then show that bounded makespan planning is equivalent to a feasible planning problem in the state-cost space.

\subsection{Bounded-Speed MRMP}
\label{sec:problem:mrmp}
 
\begin{definition}[{\mrmp} Instance] \label{def:mrmp-instance}
A {\mrmp} instance $P=\{R,O,s,G, V, \rho\}$ consists of a team of $n$ robots $R=\{r_1,\ldots,r_n\}$ moving in a shared workspace containing static obstacles $O$.
Each robot $r_i$ has configuration space {\cspacei}, free space $\cfreei\subseteq\cspacei$, start configuration $s_i\in\cfreei$, a set of goal configurations $G_i\subseteq\cfreei$, and a maximum velocity $v_i^{\mathrm{max}}$. The composite start is $s=(s_1,\ldots,s_n)$.
The composite goal set is $G=G_1\times\ldots\times G_n$. 
$V=(v_1^{\mathrm{max}},\ldots,v_n^{\mathrm{max}})$ is the set of maximum velocity constraints.
\end{definition}

The parameter $\rho$ denotes the fixed resolution used for motion validation and conflict detection.
This is the same concept described in~\cite{smsa-rmmpucs-21} extending the traditional resolution model used for interpolation in {\sbmp} to {\mrmp} and conflict detection.
All validity, feasibility, and optimality statements in this paper are with respect to this resolution model.
For brevity, we will omit further notation of $\rho$.

The composite configuration space is $\compcspace = \mathcal{C}_1 \times \cdots \times \mathcal{C}_n.$
A composite configuration \(q=(q_1,\ldots,q_n)\in\compcspace\) is valid if $q_i\in\cfreei$ for all $i$ and no pair of robots is in collision.
The set of valid composite configurations is $\compfree$.




\begin{definition}[Bounded-Speed Composite Motions]\label{def:bounded-speed-motions}
A geometric motion for robot $r_i$ is a continuous map $\sigma_i : [0,1]\rightarrow\cfreei$ with $\sigma_i(0)=q_i$ and $\sigma_i(1)=q_i'$.
Let $\Delta t_i(\sigma_i,v_i^{\mathrm{max}})$ denote the minimum duration of $\sigma_i$ while satisfying the velocity bound of robot $r_i$.
\end{definition}

A composite motion in $\compcspace$ is a tuple $e=(\sigma_1,\ldots,\sigma_n)$ connecting $q=(q_1,\ldots,q_n)$ to $q'=(q_1',\ldots,q_n')$.
Its duration is defined by the slowest individual motion:
\[
\Delta t(e,V) = \max_{i\in\{1,\ldots,n\}}\Delta t_i(\sigma_i,v_i^{\mathrm{max}})
\]
All robots execute their motions synchronously over this common duration.
If the length of $\sigma_i$ is 0, $r_i$ remains stationary during $e$.
The composite motion $e$ is valid if its synchronized execution remains in $\compfree$ for all times $t\in[0,\Delta t(e,V)].$

A composite path is a finite sequence of valid composite motions $E=\{e_1,\ldots,e_K\}$ such that $e_1$ starts at $s$, $e_K$ ends in $G$, and consecutive motions share endpoints.
Its duration is 
\[
T(E)=\sum_{k=1}^{K}\Delta t(e_k,V).
\]
The composite path $E$ induces a continuous composite trajectory $q_E:[0,T(E)]\rightarrow\compfree$.
The individual robot trajectory $\pi_i$ is the projection of $q_E$ onto $\cspacei$.




\begin{definition}[Valid team trajectory and makespan]\label{def:team-traj-and-makespan}
Let $T_i$ be the first time at which $r_i$ reaches $G_i$ and remains there.
We remove any trailing wait after all robots have reached their goals and remain there, so the terminal time of the composite trajectory satisfies
\[
T(E)=\max_{i\in\{1,\ldots,n\}} T_i.
\]
The team trajectory is $\Pi=(\pi_1,\ldots,\pi_n)$, and its cost is the team completion time, or makespan, 
\[
J(\Pi)=\max_{i\in\{1,\ldots,n\}}T_i = T(E).
\]
\end{definition}

Let $\mathcal{T}(P)$ denote the set of all trajectories for instance $P$.
The optimal makespan value is 
\[
J^*=\inf_{\Pi\in\mathcal{T}(P)}J(\Pi).
\]
If the infimum is attained, an optimal solution is any
\[
\Pi^*\in\arg\min_{\Pi\in\mathcal{T}(P)}J(\Pi).
\]



\begin{definition}[Bounded {\sc \textbf{MRMP}} feasibility]\label{def:bounded-feasibility}
For a bound $B\in\mathbb{R}_{\ge 0}$, the bounded {\mrmp} feasibility problem $P_B$ asks for a valid team trajectory $\Pi\in\mathcal{T}(P)$ such that $J(\Pi)\le B$.
\end{definition}

\subsection{State-Cost Bounded Feasibility}
\label{sec:problem:state-cost}

\begin{definition}[{\sc \textbf{MRMP}} state-cost space]
The {\mrmp} state-cost space is $Z=\compcspace\times\mathbb{R}_{\ge 0}$, $z=(q,c)$, where $c$ is the global elapsed team-time from the start. The initial state-cost point is $z_{init}=(s,0)$.    
\end{definition}

{\aox} represents optimal planning by augmenting the state with accumulated cost.
For makespan {\mrmp}, the accumulated cost is elapsed team-time, so, extending the notation from the original {\aox} work~\cite{hauser2016asymptotically}, we use a constant running cost $L(q,u)=1$ and zero terminal cost $\Phi(q)=0$ where $u$ is the control input.
Equivalently, in the continuous bounded-speed model,
\[
\dot q_i(t)=u_i(t), \qquad \|u_i(t)\|\leq v_i^{\mathrm{max}},
\]
and
\[
\dot z(t)=
\begin{bmatrix}
\dot q(t)\\
\dot c(t)
\end{bmatrix}
=
\begin{bmatrix}
u(t)\\
1
\end{bmatrix}.
\]
In the edge-based representation, traversing a composite edge $e$ updates the cost coordinate by $c'=c+\Delta t(e,V)$.

\begin{observation}[Makespan is elapsed state-cost]\label{obs:makespan-is-state-cost}
For any valid team trajectory $\Pi$ induced by a composite path with terminal time $T$,
\[
C(\Pi)
=
\int_0^T L(q(t),u(t))\,dt+\Phi(q(T))
=
\int_0^T 1\,dt
=
T.
\]
Thus, with $T=J(\Pi)$, minimizing makespan is equivalent to minimizing {\aox} trajectory cost
with $L\equiv 1$ and $\Phi\equiv 0$.
\end{observation}

For a bound $B$, define the bounded state-cost goal set
\[
G_B=\{(q,c)\in\compfree\times\mathbb{R}_{\ge 0} | q\in G, c\le B\}.
\]

\begin{proposition}[Bounded {\mrmp} and bounded state-cost feasibility]\label{prop:bounded-mrmp-and-state-cost-feasibility}
    A team trajectory $\Pi$ solves the bounded {\mrmp} feasibility problem $P_B$ if and only if its lifted state-cost trajectory starts at $(s,0)$ and reaches $G_B$.
\end{proposition}

\begin{proof}
    By Observation~\ref{obs:makespan-is-state-cost}, the cost coordinate $c$ equals elapsed team-time, and, by Definition~\ref{def:team-traj-and-makespan}, elapsed team-time equals the makespan $J(\Pi)$.
    Therefore reaching $q\in G$ with $c\le B$ is exactly equivalent to reaching all robot goals with $J(\Pi)\le B$.
\end{proof}
These definitions put bounded-speed {\mrmp} in the {\aox} form.
The composite state is $q$, the accumulated cost coordinate $c$ is global elapsed team-time, and the running cost is $L\equiv1$.
Therefore, solving $P_B$ is equivalent to solving a feasible planning problem in $Z$ with bounded goal set $G_B$.
Section~\ref{sec:method} uses this equivalence by repeatedly calling {\barc} under the current incumbent bound $B$.

\section{AO-ARC}
\label{sec:method}


In this section, we present {\aoarc}, an anytime {\mrmp} algorithm that applies the {\aox} meta-algorithm to the bounded-speed {\mrmp} problem defined in Section~\ref{sec:problem}.
The {\aox} layer (Alg.~\ref{alg:aoarc}) maintains an incumbent bound $B$ and iteratively calls {\barc} to attempt to solve the bounded feasibility problem $P_B$.
Whenever a returned trajectory satisfies $J(\Pi)<B$, {\aoarc} updates the incumbent and tightens the bound for the next call.

{\barc} adapts the individual path queries and the local conflict resolution subproblem strategies of {\arc} to the bounded state-cost space, using a bounded {\rrtc}~\cite{wilson2025aorrtc} as the planner for both.
The details of {\barc} are presented in Section~\ref{sec:method:bounded-arc}, and the theoretical analysis of {\aoarc} is presented in Section~\ref{sec:method:theoretical}.

\begin{figure*}[t]
    \centering
    \hfill
    \begin{subfigure}[t]{0.245\textwidth}
        \centering
        \includegraphics[width=\linewidth]{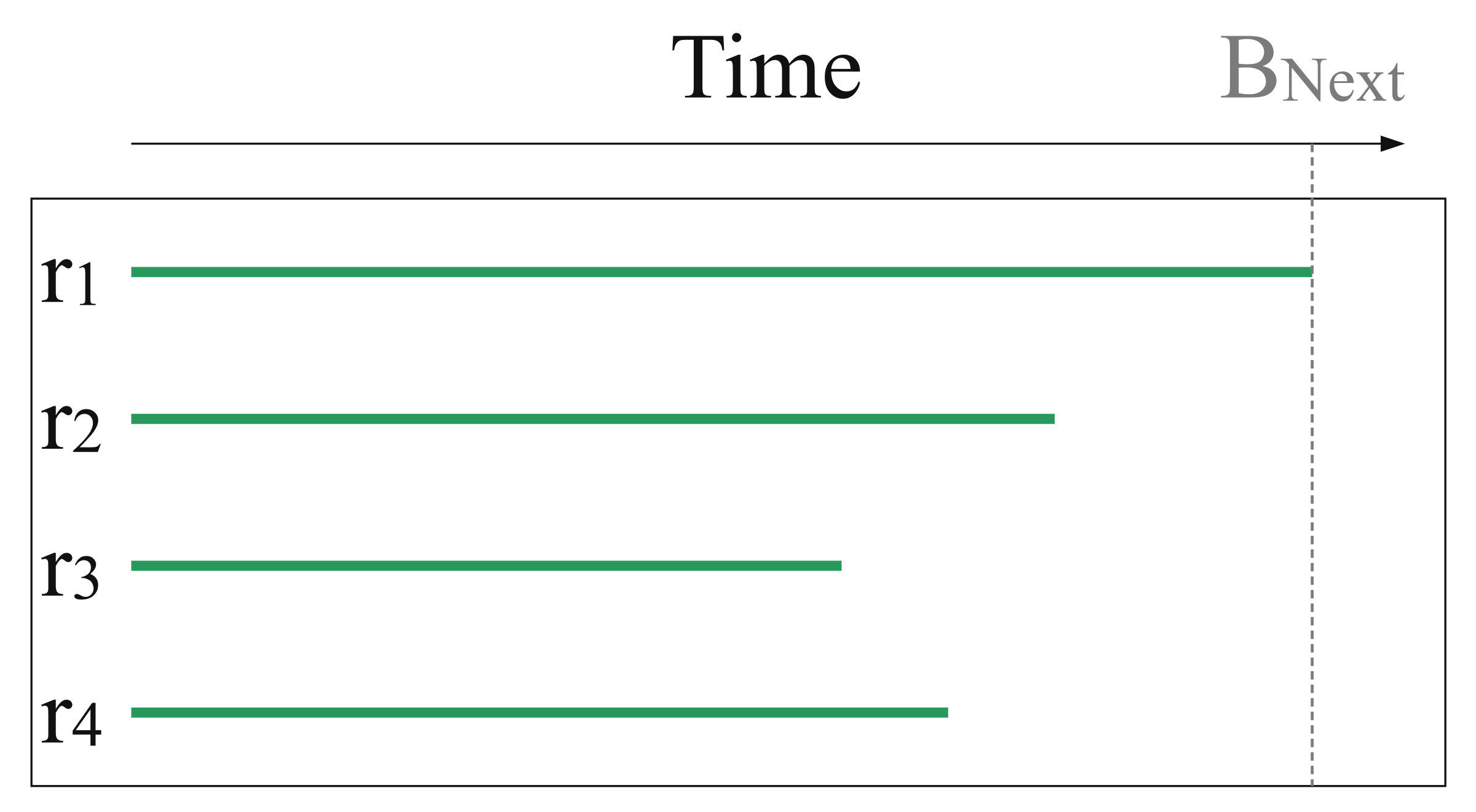}
        \caption{Derive New Bound}
        \label{fig:bounded-arc:new-bound}
    \end{subfigure}
    \begin{subfigure}[t]{0.245\textwidth}
        \centering
        \includegraphics[width=1.0\linewidth]{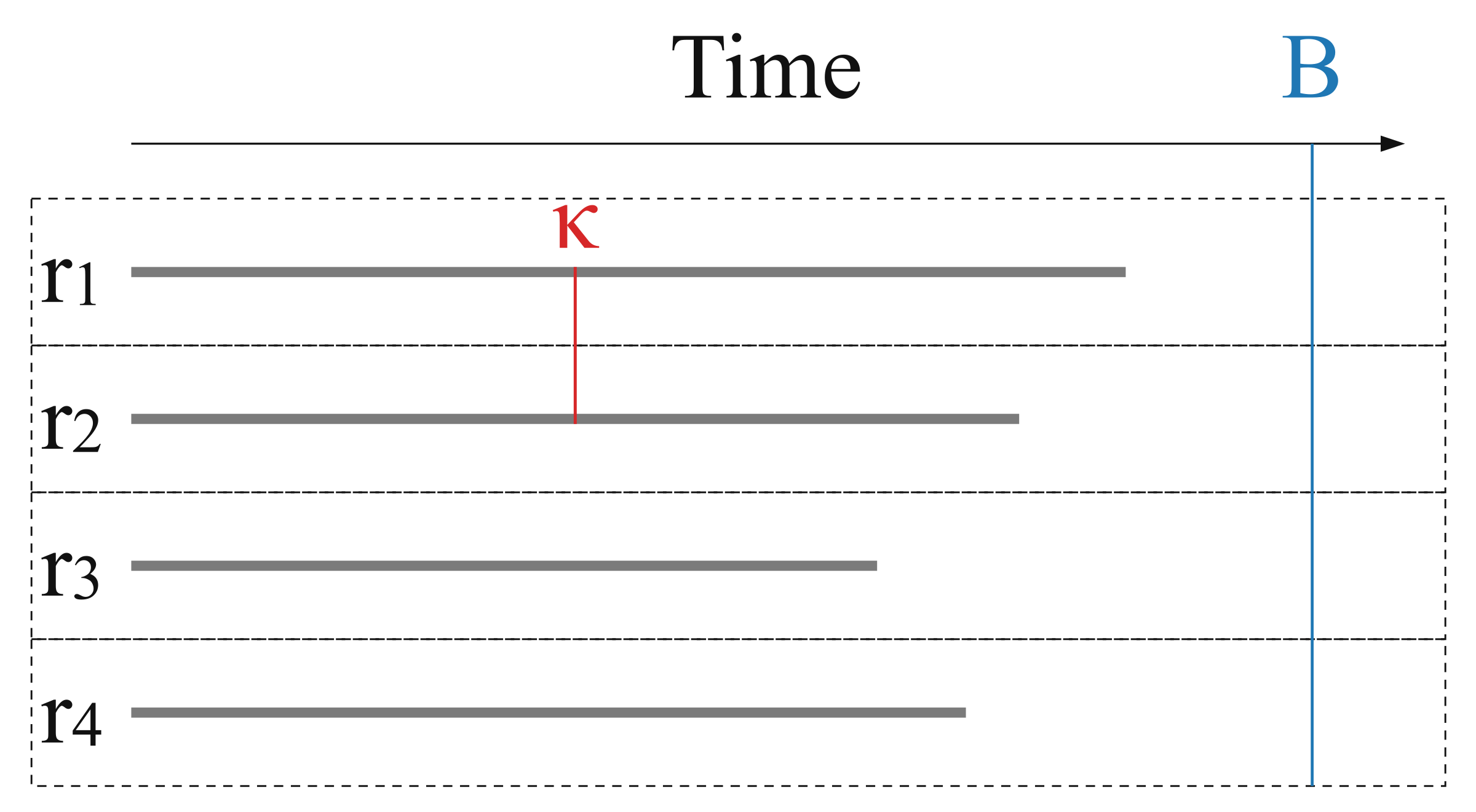}
        \caption{Plan Individual Paths}
        \label{fig:bounded-arc:ind-paths}
    \end{subfigure}
    \begin{subfigure}[t]{0.245\textwidth}
        \centering
        \includegraphics[width=1.0\linewidth]{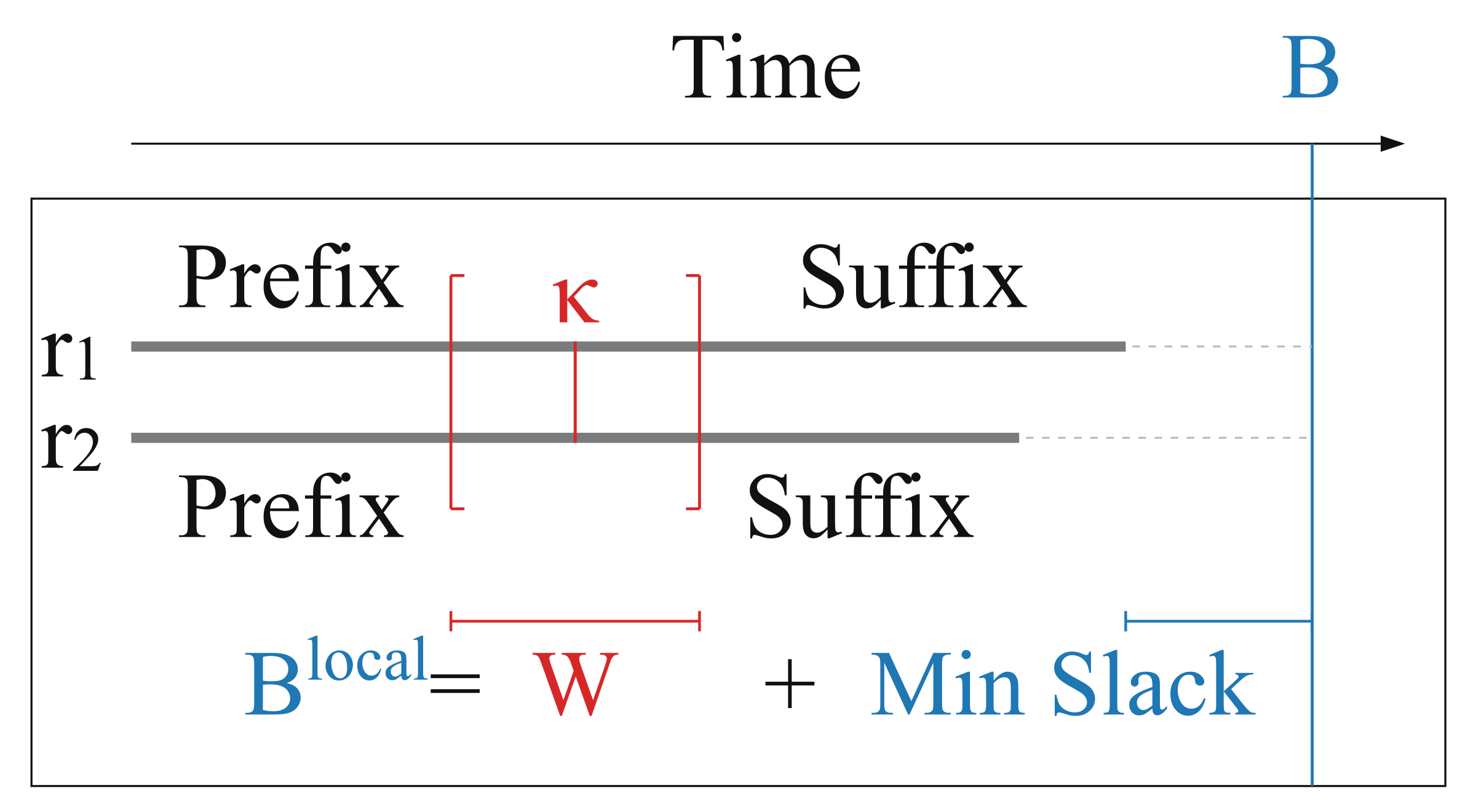}
        \caption{Create Subproblem}
        \label{fig:bounded-arc:subproblem}
    \end{subfigure}
    \begin{subfigure}[t]{0.245\textwidth}
        \centering
        \includegraphics[width=1.0\linewidth]{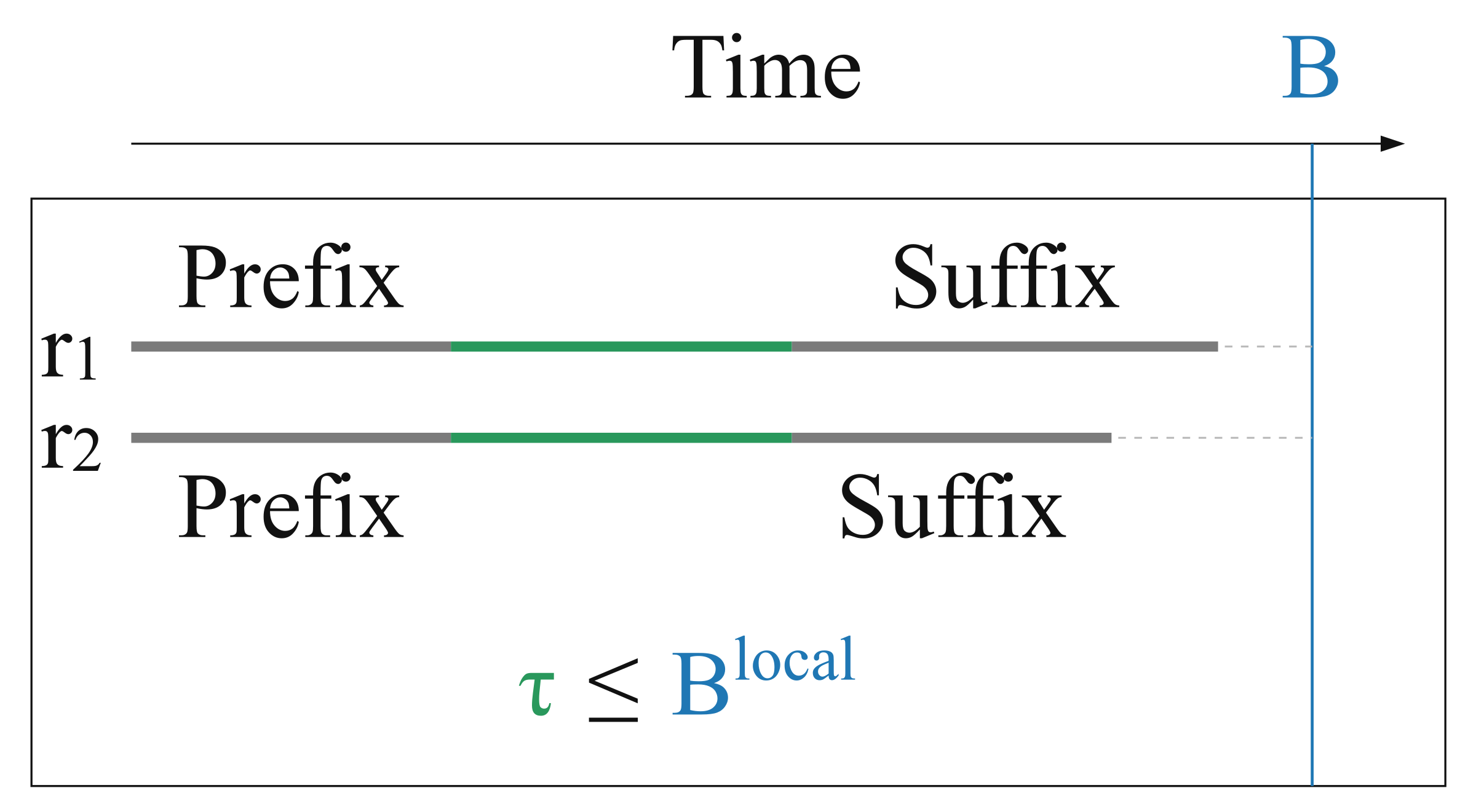}
        \caption{Patch Paths}
        \label{fig:bounded-arc:resolution}
    \end{subfigure}
    \caption{
    \small
    {\barc} overview: 
    (a) The new bound is derived from the makespan of the previous solution.
    (b) Each robot is planned for individually with respect to the new makespan bound. A conflict $\kappa$ (red line) exists between robots 1 and 2.
    (c) A subproblem with a time window $W$ is defined around $\kappa$ creating a prefix and suffix segment on either side for both robots. The local bound $B^{\mathrm{local}}$ is derived from the window size and remaining slack before the team-time bound.
    (d) The repair segment of duration $\tau\le B^{\mathrm{local}}$ is used to update the robot trajectories, keeping their total duration under the team-time bound $B$.
    The process is repeated until no conflicts remain.
    }
    \label{fig:bounded-arc}
    \vspace{-2em}
\end{figure*}

\vspace{-0.7em}
\begin{algorithm}[]
\caption{\textsc{AO-ARC}}
\begin{algorithmic}[1]
\small
\Require {\mrmp} instance $P=\{R,O,s,G, V\}$, Bounded feasibility solver $F$, Initial window size $W$, Fair budget schedule $\beta_1,\beta_2,\beta_3,\ldots$
\Ensure Best team trajectory $\Pi_{\mathrm{best}}$ found

\State $\Pi_{\mathrm{best}} \gets\textsc{ARC}(P,F,W)$\label{alg:aoarc:set-init-traj}
\If{$\Pi_{\mathrm{best}}=\emptyset$} \Return $\emptyset$\EndIf
\State $B\gets J(\Pi_{\mathrm{best}})$\label{alg:aoarc:set-init-bound}

\For{$k = 1,2,3\ldots$ until termination}\label{alg:aoarc:main-loop-start}
    \State $\Pi \gets \textsc{BoundedARC}(P, B, F, W, \beta_k)$\label{alg:aoarc:bounded-arc}
    \If{$\Pi \neq \emptyset$ \textbf{and} $J(\Pi) < B$}\label{alg:aoarc:check-bounded-arc}
        \State $\Pi_{\mathrm{best}} \gets \Pi$\label{alg:aoarc:update-solution}
        \State $B \gets J(\Pi_{\mathrm{best}})$\label{alg:aoarc:update-bound}
    \EndIf
\EndFor\label{alg:aoarc:main-loop-end}

\State \Return $\Pi_{\mathrm{best}}$\label{alg:aoarc:return}
\end{algorithmic}
\label{alg:aoarc}
\end{algorithm}

\vspace{-1.4em}
\subsection{Bounded ARC}
\label{sec:method:bounded-arc}

{\barc} (Alg.~\ref{alg:bounded-arc}) attempts to solve the bounded feasibility problem induced by $(P,B)$.
It first plans independently for each robot using a bounded feasibility solver $F$ with the team-time bound $B$ (lines~\ref{alg:barc:ind-path-start}-\ref{alg:barc:ind-path-end}).
Thus each individual trajectory $\pi_i$ reaches $G_i$ within the current makespan bound.
If any individual bounded query fails within the current budget, the call returns $\emptyset$.
We use the bounded {\rrtc} routine from {\aorrtc}~\cite{wilson2025aorrtc} for $F$; the analysis applies to any bounded feasibility solver satisfying the assumptions in Section~\ref{sec:method:theoretical}.


The independently computed trajectories are then checked for inter-robot conflicts (lines~\ref{alg:barc:find-conflict-outer}, \ref{alg:barc:find-conflict-inner}).
A conflict is $\kappa=\{r_i,r_j,t\}$, where robots $r_i$ and $r_j$ collide at time $t$ along the current team trajectory $\Pi$.
To repair $\kappa$, {\sc SolveSubproblem} (Alg.~\ref{alg:solve-subproblem}) constructs an expansion hierarchy $H$ from $\kappa$ and the initial window size $W$.
Each level $h\in H$ defines a local subproblem $P_h^{\mathrm{local}}$ by selecting a local robot set $R_h^{\mathrm{local}}$, a time interval $I_h=[t_h^-,t_h^+]$ around the conflict, and local $\cspace$ bounds for the robots in $R_h^{\mathrm{local}}$.
Initially, $R_h^{\mathrm{local}}$ contains the conflicting robots and $I_h$ is centered at $t$ with width $W$.
For each $r_i\in R_h^{\mathrm{local}}$, the local start and goal are the current configurations $\pi_i(t_h^-)$ and $\pi_i(t_h^+)$.
The local subproblem requires the robots in $R_h^{\mathrm{local}}$ to move between these local endpoints while avoiding static obstacles and mutual collisions within the local robot set.

If $F$ returns a local trajectory $\Pi_h^{\mathrm{local}}$, then {\sc PatchSolution} replaces the corresponding trajectory segments of the robots in $R_h^{\mathrm{local}}$ with $\Pi_h^{\mathrm{local}}$.
If the repair duration differs from the original window duration, the suffixes of the patched robots are shifted in time while preserving their geometric paths.
Robots outside $R_h^{\mathrm{local}}$ are unchanged.
The patched team trajectory is then checked again for conflicts, and the process repeats until no conflicts remain or the budget is exhausted.

\begin{algorithm}[]
\caption{\textsc{BoundedARC}}
\begin{algorithmic}[1]
\small
\Require {\mrmp} instance $P=\{R,O,s,G,V\}$, Bound $B$, Bounded feasibility solver $F$, Window size $W$, Planning budget $\beta$
\Ensure Team trajectory $\Pi$ such that $J(\Pi)\leq B$
\State $\Pi\gets\emptyset$\label{alg:barc:init}
\For{$r_i\in R$}\Comment{Compute bounded independent paths}\label{alg:barc:ind-path-start}
    \State $\pi_i\gets F(r_i,s_i,G_i,O, B, \beta)$\label{alg:barc:ind-path-solve}\Comment{Call bounded planner}
    \If{$\pi_i==\emptyset$}\label{alg:barc:ind-path-check}
        \State\Return$\emptyset$\label{alg:barc:ind-path-fail}
    \EndIf
    \State$\Pi\gets\Pi\cup\{\pi_i\}$\label{alg:barc:ind-path-add}
\EndFor\label{alg:barc:ind-path-end}
\State $\kappa=\{r_i,r_j,t\}\gets$ {\sc FindFirstConflict($\Pi$)}\label{alg:barc:find-conflict-outer}
\While{$\kappa\neq\emptyset$ and budget $\beta$ remains}\label{alg:barc:conflict-resolution-loop-start}
    \State$(P^{\mathrm{local}},\Pi^{\mathrm{local}})\gets${\sc SolveSub($F,\kappa,P,\Pi,W,B,\beta$)}\label{alg:barc:solve-subproblem}
    \If{$\Pi^{\mathrm{local}}==\emptyset$}
        return $\emptyset$
    \EndIf
    \State $\Pi\gets${\sc PatchSolution($\Pi,\mathcal{P}^{\mathrm{local}},\Pi^{\mathrm{local}}$)}\label{alg:barc:apply-patch}
    \State $\kappa=\{r_i,r_j,t\}\gets$ {\sc FindFirstConflict($\Pi$)}\label{alg:barc:find-conflict-inner}
\EndWhile\label{alg:barc:conflict-resolution-loop-end}
\If{$\kappa\neq\emptyset$}\label{alg:barc:check-success}
    \State\Return$\emptyset$\label{alg:barc:fail}
\EndIf
\State \Return $\Pi$\label{alg:barc:return}
\end{algorithmic}
\label{alg:bounded-arc}
\end{algorithm}

\begin{algorithm}[]
\caption{\textsc{SolveSubproblem (SolveSub)}}
\begin{algorithmic}[1]
\small
\Require Bounded feasibility solver $F$, Conflict $\kappa=\{r_i,r_j,t\}$, {\mrmp} instance $P$, Current team trajectory $\Pi$, Initial window $W$, Bound $B$, Planning budget $\beta$
\Ensure Local subproblem $P^{\mathrm{local}}$, local trajectory $\Pi^{\mathrm{local}}$
\State Initialize expansion hierarchy $H$ from $\kappa$ and $W$
\For{each expansion level $h\in H$}
    \State $P^{\mathrm{local}}_{h}\gets${\sc CreateExpSubprob($\kappa, P, \Pi, h$)}
    \State $B^{\mathrm{local}}_{h}\gets${\sc ComputeLocalBound($P^{\mathrm{local}}_{h}, \Pi, B$)}
    \If{$h\neq h_{\mathrm{final}}$}
        \If{$B^{\mathrm{local}}_{h}\le\epsilon$} continue\EndIf
        \State $B^{\mathrm{local}}_{h}\gets B^{\mathrm{local}}_h-\epsilon$\Comment{Enforce well-behaved property}\label{alg:well-behaved}
    \EndIf
    \State $\Pi^{\mathrm{local}}_{h}\gets F(P^{\mathrm{local}}_{h},B^{\mathrm{local}}_{h},\beta_{h})$
    \If{$\Pi^{\mathrm{local}}_{h}\neq\emptyset$}
        \State\Return $(P^{\mathrm{local}}_{h},\Pi^{\mathrm{local}}_{h})$
    \EndIf
\EndFor
\State\Return$\emptyset$
\end{algorithmic}
\label{alg:solve-subproblem}
\end{algorithm}

\subsubsection{Local Subproblem Bounds}
\label{sec:method:local-bounds}


The key modification in {\aoarc} is that every local repair is given a duration bound that preserves the global makespan bound $B$.
Consider a subproblem $P^{\mathrm{local}}_{h}$ with robot set $R^{\mathrm{local}}$ and repair interval $I_h=[t_h^-,t_h^+]$.
For each robot $r_i\in R^{\mathrm{local}}$, let $\lambda^-_i$ be the duration of its fixed prefix before $t_h^-$, and let $\lambda^+_i$ be the duration of its fixed suffix after $t_h^+$.
If the synchronized repair has duration $\tau$, then the patched duration of robot $r_i$ is $\lambda^-_i+\tau+\lambda^+_i$.
Therefore, to preserve $J(\Pi)\le B$, the local repair must satisfy $\tau\leq B-\lambda^-_i-\lambda^+_i$ for every robot in $R_h^{\mathrm{local}}$.
{\aoarc} therefore uses the local bound $
B^{\mathrm{local}}_h=\min_{r_i\in R_h^{\mathrm{local}}}(B-\lambda_i^--\lambda_i^+).
$

This bound is conservative because it requires a single synchronized repair duration for all robots in $R^{\mathrm{local}}_h$.
It may reject asynchronous repairs that would still satisfy the global makespan bound.
This conservatism does not affect the completeness of {\barc} in the limit because failed bounded repairs cause {\sc SolveSubproblem} to continue through the expansion hierarchy.





\subsubsection{Solving Subproblems}
\label{sec:method:subproblem-expansion}

The local bound ensures that an accepted repair preserves the current global bound, but it does not by itself remove {\arc}'s local decision limitation.
A repair may require {\cspace} regions or a larger time interval than is included in the initial local subproblem.
{\aoarc} handles this by expanding the same subproblem before declaring the bounded repair attempt failed.
As $h$ increases, the local time interval is enlarged and the {\cspace} bounds are relaxed.
If a patched segment later causes a conflict with a robot outside $R_h^{\mathrm{local}}$, the next repair includes the robots that generated the patched segment together with the new conflicting robot.

For nonterminal expansion levels $h\neq h_{\mathrm{final}}$, {\sc SolveSubproblem} uses the reduced bound $B_h^{\mathrm{local}}-\epsilon$.
This forces strict improvement relative to the current bound when a nonterminal local repair is accepted, which is used in the well-behavedness argument in Section~\ref{sec:method:theoretical}.
If $B_h^{\mathrm{local}}\le\epsilon$, the nonterminal query is skipped.
The reduction is not applied at the terminal level $h_{\mathrm{final}}$.

At the terminal level, the time interval spans the full current trajectories of the robots in $R_h^{\mathrm{local}}$, and the local starts and goal coincide with their global starts and goals.
The local {\cspace} bounds are removed, so the local query is the full composite bounded problem for the current local robot set.
In this case $\lambda_i^-=\lambda_i^+=0$ and $B_h^{\mathrm{local}}=B$.
Repeated conflict detection and robot-set expansion can include all robots, at which point the terminal query is the full composite bounded {\mrmp} problem under the current makespan bound $B$.

For the theoretical properties in Section~\ref{sec:method:theoretical}, nonterminal levels receive finite effort so that the expansion can occur, while the terminal level receives unbounded effort under the fair budget schedule.

\subsection{Theoretical Properties}
\label{sec:method:theoretical}

In this section, we show that {\aoarc} is an instance of the {\aox} meta-algorithm~\cite{hauser2016asymptotically} applied to bounded-speed {\mrmp} with the makespan objective as defined in Section~\ref{sec:problem}.
The proof uses the state-cost equivalence established in Section~\ref{sec:problem:state-cost}.
For makespan {\mrmp}, the accumulated cost coordinate is global elapsed team-time with $L\equiv1$ and $\Phi\equiv0$.
Therefore, for any valid team trajectory $\Pi$
\[
C(\Pi)=\int^{T}_{0}1dt=T=J(\Pi).
\]
Thus by Proposition~\ref{prop:bounded-mrmp-and-state-cost-feasibility}, solving the bounded {\mrmp} feasibility problem $P_B$ is equivalent to finding a feasible trajectory in the state-cost space $Z=\compcspace\times\mathbb{R}_{\ge0}$ from $(s,0)$ to the bounded goal set $G_B$.
This is the bounded state-cost feasibility problem used by {\aox}~\cite{hauser2016asymptotically}.


\begin{assumption}[Well-behaved bounded feasibility solver]\label{assumption:solver}
The solver $F$ is a probabilistically complete bounded state-cost feasibility solver
and is well-behaved in the {\aox} sense.
All calls to $F$ use independent randomness.
For any problem $P_B$ with optimum $J_{P_B}^*$, $F$ returns a feasible solution $C\le B$ with probability one in the infinite budget limit, and there exists $w_F\in(0,1]$ such that 
\[\mathbb{E}[C_{P_B}\mid B]-J_{P_B}^*\le (1-w_F)(B-J_{P_B}^*).\]
\end{assumption}

\begin{assumption}[Fair expansion and budget allocation]\label{assumption:fair}
{\sc SolveSubproblem} uses a fair expansion hierarchy.
Nonterminal expansion levels $h\ne h_{\mathrm{final}}$ are queried with the reduced bound $B_h^{\mathrm{local}}-\epsilon$, for fixed $\epsilon>0$, and receive finite effort.
The terminal level $h_{\mathrm{final}}$ uses the unreduced bound $B_h^{\mathrm{local}}$ and receives unbounded effort.
The finite nonterminal attempts give a nonzero probability that a {\barc} call reaches a terminal/global-equivalent bounded query before accepting a nonterminal repair.
Across repeated conflict detection/expansion cycles, if local repairs do not yield a globally valid bounded solution, the policy eventually reaches the full composite query with unbounded effort.
\end{assumption}




\begin{lemma}[Local-bound soundness]\label{lemma:local_bound_soundness}
Let $P^{\mathrm{local}}_{h}$ be a local subproblem involving robots $R^{\mathrm{local}}\subseteq R$.
For each $r_i\in R^{\mathrm{local}}$, let $\lambda_i^-$ and $\lambda_i^+$ be the fixed prefix and suffix durations outside the repair window, and define
\[
    B_h^{\mathrm{local}} =
    \min_{r_i\in R^{\mathrm{local}}}
    \left.(B-\lambda_i^- - \lambda_i^+\right).
\]
If $h\ne h_{\mathrm{final}}$, {\sc SolveSubproblem} calls $F$ with bound $B_h^{\mathrm{local}}-\epsilon.$
If $h=h_{\mathrm{final}}$, it calls $F$ with bound $B_{h_{\mathrm{final}}}^{\mathrm{local}}=B.$
Any returned local repair preserves the global bound $B$.
\end{lemma}

\begin{proof}
    As $B_h^{\mathrm{local}}-\epsilon<B_h^{\mathrm{local}}$,
    for every $r_i\in R^{\mathrm{local}}$, $\tau\le B_h^{\mathrm{local}}\le B-\lambda^-_i-\lambda_i^+$ for every $h\in H$.
    Thus the repaired duration of robot $r_i$ always satisfies $\lambda_i^-+\tau+\lambda_i^+ \le B$.
    Robots outside of $R^{\mathrm{local}}$ are unchanged by the patch, so the global bound is preserved.
\end{proof}

\begin{lemma}[{\sc \textbf{BoundedARC}} is a bounded feasible planner]\label{lemma:bounded-arc-feasible-planner}
Under Assumptions~\ref{assumption:solver} and~\ref{assumption:fair}, {\barc} is a sound and probabilistically complete planner for $P_B$.
\end{lemma}

\begin{proof}
Soundness follows because {\barc} only returns after {\sc FindFirstConflict} reports no conflicts, and every accepted path preserves $J(\Pi)\le B$ by Lemma~\ref{lemma:local_bound_soundness}.
Therefore any returned trajectory is valid and solves $P_B$.


For probabilistic completeness, suppose $B>J^*$.
Then $P_B$ is feasible.
The $\epsilon$-reduction is applied only to $h\ne h_{\mathrm{final}}$ subproblem levels, and $h_{\mathrm{final}}$ retains the global bound $B$.
If lower levels do not produce a valid team trajectory, Assumption~\ref{assumption:fair} ensures that the subproblem hierarchy recovers the full composite bounded problem with unbounded effort.
At that level, the query is exactly the bounded state-cost feasibility problem equivalent to $P_B$.
By Assumption~\ref{assumption:solver}, this feasible query is solved with probability approaching one.
\end{proof}



\begin{lemma}[{\sc BoundedARC} is well-behaved]\label{lemma:well-behaved}
Let $B_0<\infty$ be the first finite incumbent found by {\aoarc}.
For every later call with $J^*<B\le B_0$, {\barc} satisfies the {\aox} contraction condition.
\end{lemma}

\begin{proof}
Let $T_i$ be the final duration of robot $r_i$in $\Pi$.
For each robot, inspect the last event that modified its trajectory.
This event gives a slack certificate $u_i\in[0,1]$ such that $B-T_i\ge u_i(B-J^*)$ with conditional expectation at least a uniform constant $w_c>0$.

If the last event is an initial individual solve or a terminal repair, let $C$ be the makespan returned by that call to $F$, and let $J_F^*$ be the optimum of the corresponding single robot or terminal local problem.
Projecting a globally optimal team trajectory onto the robot or robots in that query gives a feasible solution for the query, so $J_F^*\le J^*$.
By Assumption~\ref{assumption:solver}, the call to $F$ has a slack certificate $u=\frac{B-C}{B-J_F^*}$ with conditional expectation at least $w_F$.
For an initial individual solve, $T_i=C$.
For a terminal repair, every involved robot has $T_i\le C$.
Therefore, in either case, $B-T_i\ge u(B-J^*)$.

If the last event is a nonterminal repair, the local query uses the reduced bound $B_h^{local}-\epsilon$.
Since $B_h^{local}\le B-\lambda_i^- - \lambda_i^+$, the patched total duration of $r_i$ satisfies $T_i\le B-\epsilon$.
For all calls after the first finite incumbent, $B\le B_0$,
so this is a deterministic contraction certificate with $u_i=w_\epsilon=\min(1,\frac{\epsilon}{B_0-J^*})>0$.
If $B_0=J^*$, convergence has already been reached.

Thus every robot has a slack certificate with conditional expectation at least $w_c=\min\{w_F,w_\epsilon\}$.
There are at most $n$ distinct certificates. 
Since disjoint calls to $F$ use fresh randomness, the standard product-slack argument gives $\mathbb{E}[\min_i u_i]\ge w_c^n$: namely, $\min_i u_i\ge\prod_i u_i$, and each distinct factor has conditional expectation at least $w_c$.

Finally, because the objective is makespan, $J(\Pi)=\max_i T_i$.
Hence the slack of the team solution is the minimum robot slack: $B-J(\Pi)=\min_i(B-T_i)\ge (B-J^*)\min_i u_i$.
Taking expectation gives 
\[
\mathbb{E}[J(\Pi)\mid B]-J^*\le (1-w_c^n)(B-J^*).
\]
Therefore, {\barc} satisfies the {\aox} contraction condition.
\end{proof}

\begin{theorem}[Almost-sure asymptotic optimality of {\sc \textbf{AO-ARC}}]\label{theorem:aoarc-ao}
Under Assumptions~\ref{assumption:solver} and \ref{assumption:fair}, {\aoarc} is almost-surely asymptotically optimal.
If $B_k$ is the incumbent makespan after the $k$th successful improvement, then 
\[
    \Pr\!\left(\lim_{k\to\infty} B_k = J^*\right)=1.
\]
\end{theorem}

\begin{proof}
By Proposition~\ref{prop:bounded-mrmp-and-state-cost-feasibility}, bounded makespan {\mrmp} is equivalent to bounded feasible planning in state-cost space.
By Lemma~\ref{lemma:bounded-arc-feasible-planner}, {\barc} is a sound probabilistically complete feasible planner for these bounded problems.
The initial call to {\arc} is used to obtain the
finite incumbent $B_0$.
By Lemma~\ref{lemma:well-behaved}, {\barc} satisfies the {\aox} well-behaved contraction condition~\cite{hauser2016asymptotically} for all subsequent finite bound calls with $J$ replacing the generic cost function.
Alg.~\ref{alg:aoarc} is therefore {\aox} with $A=$ {\barc} and cost $J$.
The {\aox} convergence theorem~\cite{hauser2016asymptotically} implies that the incumbent cost converges almost surely to the optimal cost.
Hence $B_k\to J^*$ almost surely. 
\end{proof}

\section{Evaluation}
\label{sec:evaluation}

\begin{figure*}[t]
    \centering
    \hfill
    \begin{subfigure}[t]{0.15\textwidth}
        \centering
        \includegraphics[width=\linewidth]{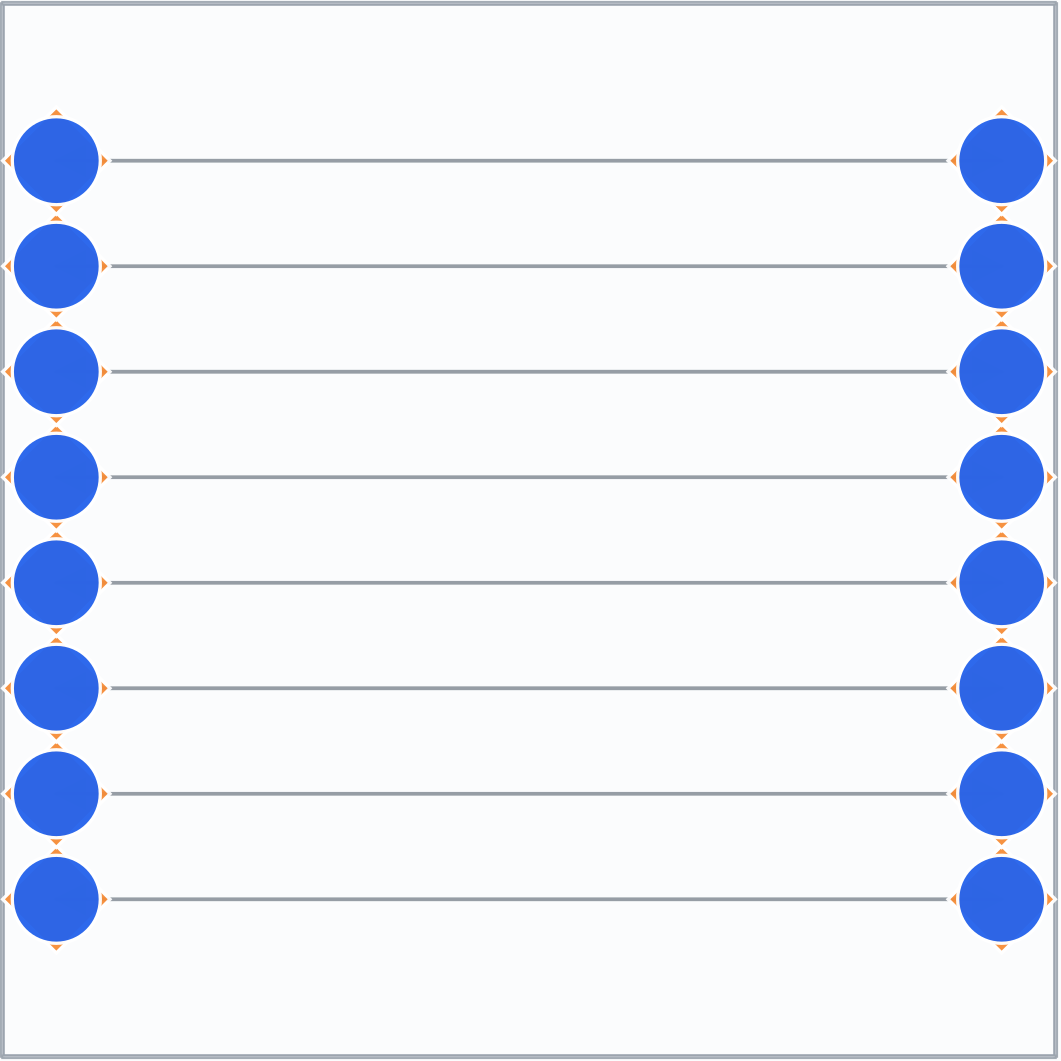}
        \caption{2D Mobile Cross}
        \label{fig:scenarios-2d-cross}
    \end{subfigure}
    \hfill
    \begin{subfigure}[t]{0.15\textwidth}
        \centering
        \includegraphics[width=1.0\linewidth]{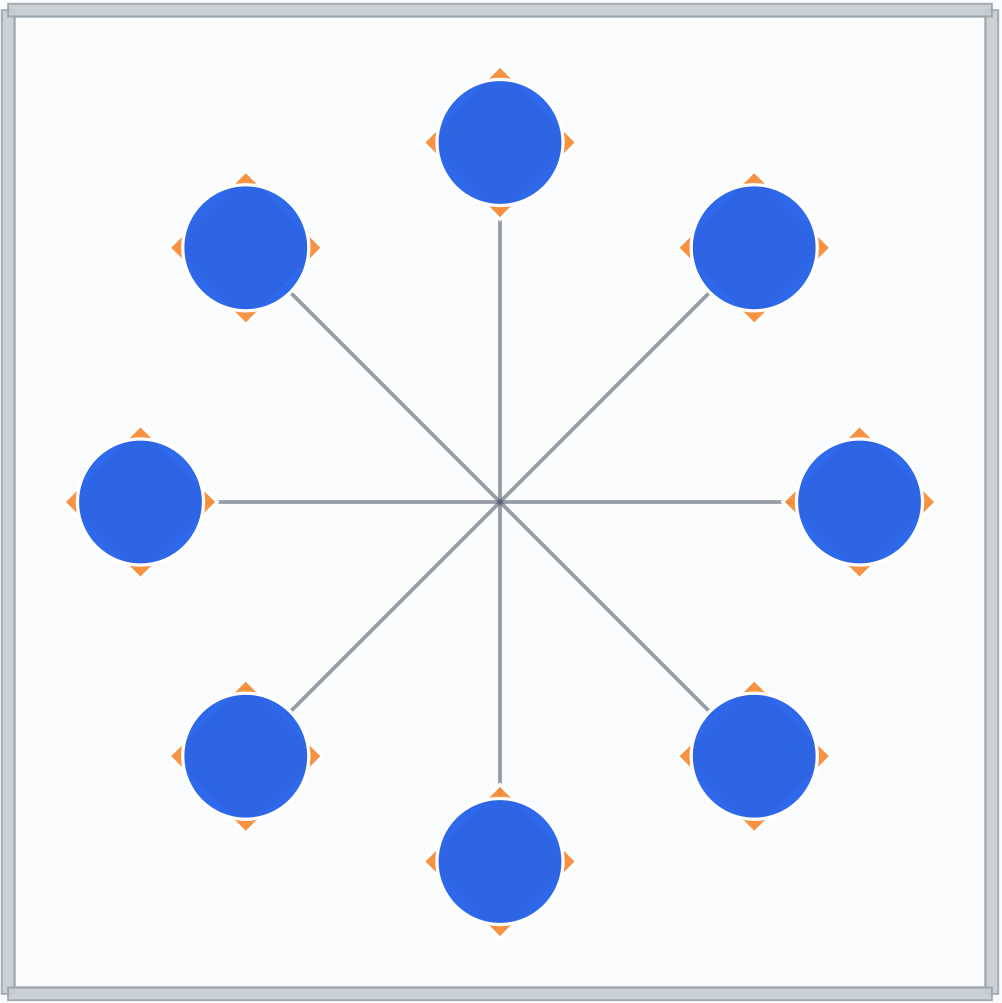}
        \caption{2D Mobile Circle}
        \label{fig:scenarios-2d-circle}
    \end{subfigure}
    \hfill
    \begin{subfigure}[t]{0.3\textwidth}
        \centering
        \includegraphics[width=1.0\linewidth]{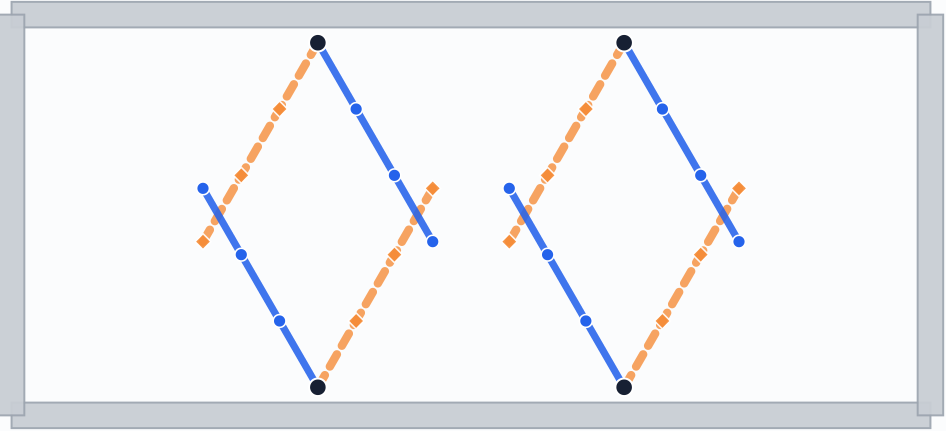}
        \caption{Planar Manipulators}
        \label{fig:scenarios-planar-cross}
    \end{subfigure}
    \hfill
    \begin{subfigure}[t]{0.15\textwidth}
        \centering
        \includegraphics[width=1.0\linewidth]{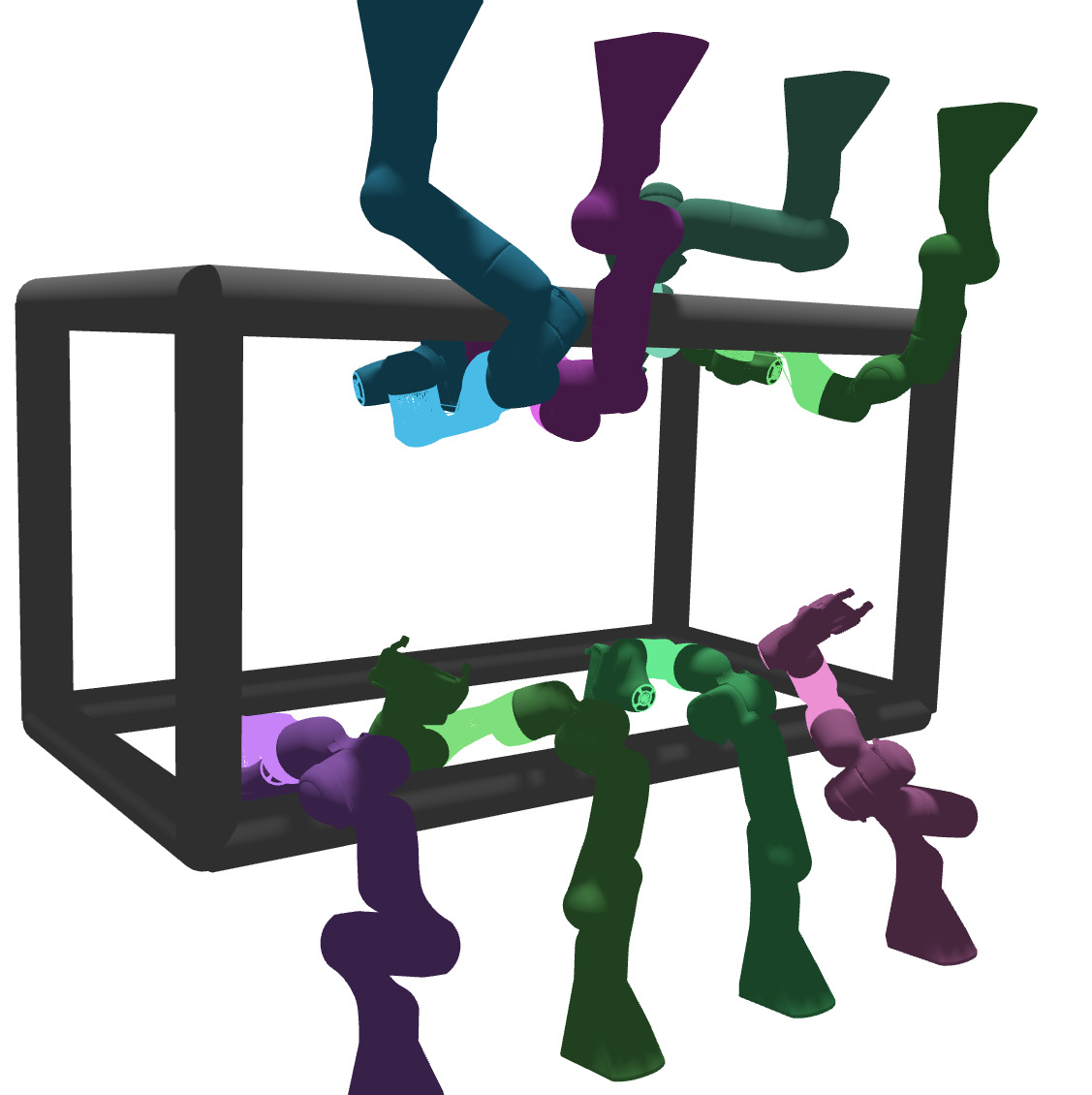}
        \caption{3D Manipulator Cage}
        \label{fig:panda}
    \end{subfigure}
    \caption{
    \small
    (a-b) 2D mobile robot scenarios with robots swapping places with their opposite counterparts. (a) Induces $\frac{n}{2}$ conflicts with two robots each. (b) Forces {\arc} variants to grow one large subproblem that includes all robots.
    (c) Rows of planar manipulators must move from the blue configurations to the orange ones, forcing conflicts that ripple to other robots in the constrained environment.
    (d) 3D Panda manipulators in a floor-ceiling mounting configuration and operating in a shared workspace (similar to automobile assembly lines).
    }
    \label{fig:scenarios}
    \vspace{-0.5em}
\end{figure*}

\begin{figure*}[t]
    \centering
    \begin{subfigure}[t]{0.45\textwidth}
        \centering
        \includegraphics[width=\linewidth]{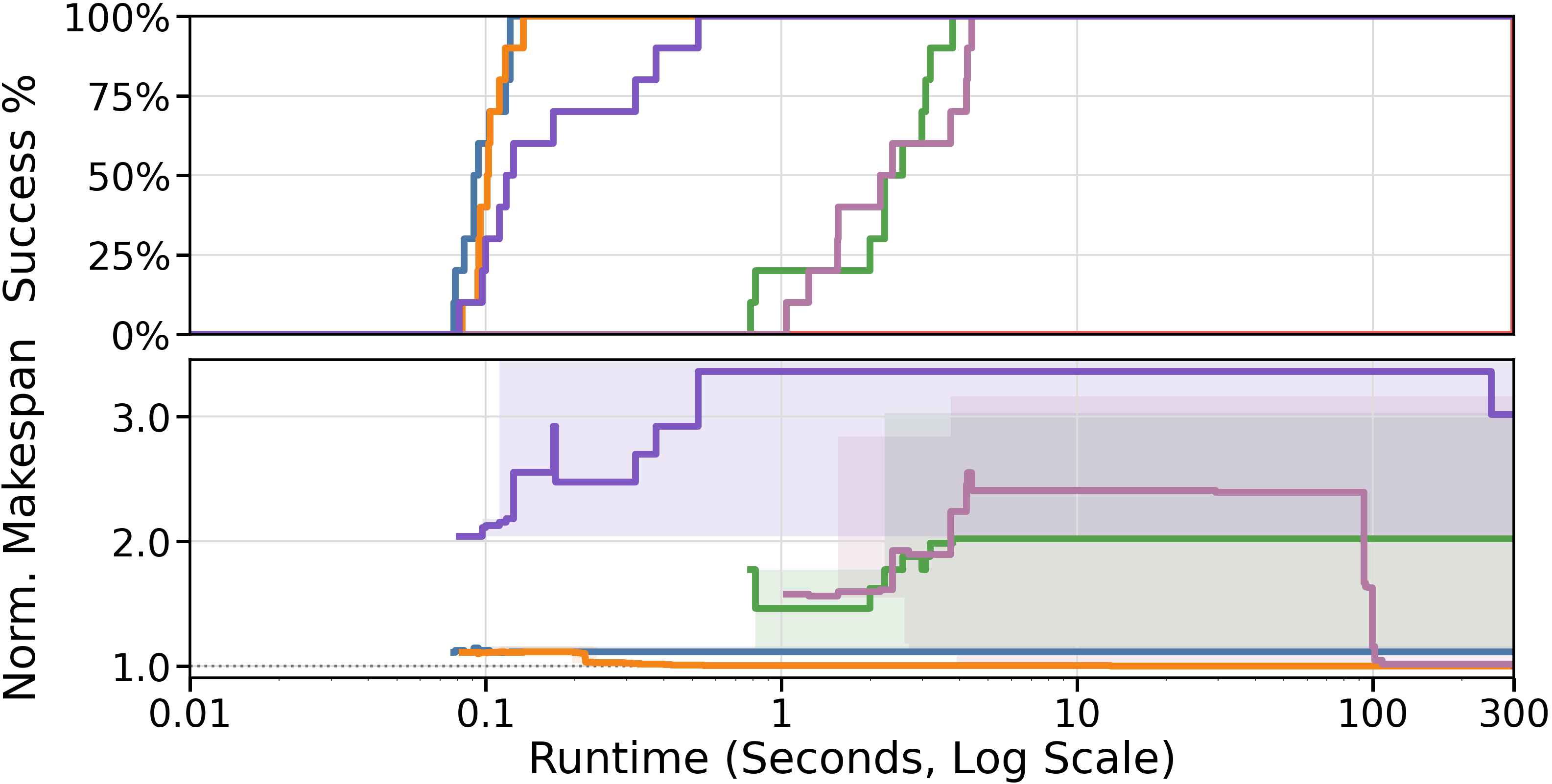}
        \caption{Cross - 8}
        \label{fig:results:parallel}
    \end{subfigure}
    \hfill
    \begin{subfigure}[t]{0.45\textwidth}
        \centering
        \includegraphics[width=\linewidth]{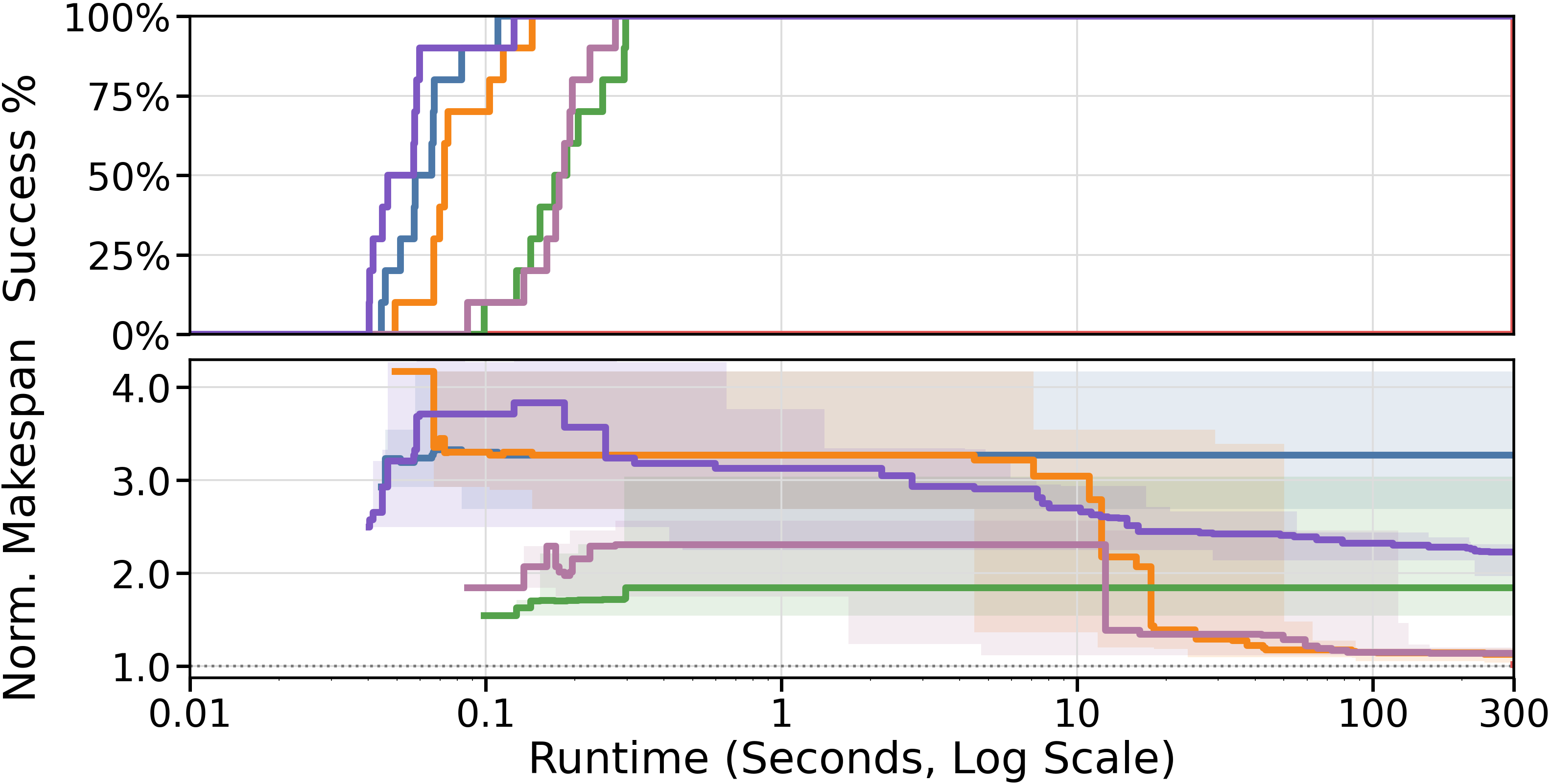}
        \caption{Circle - 8 Robots}
        \label{fig:results:circle}
    \end{subfigure}
    \\
    \begin{subfigure}[t]{0.45\textwidth}
        \centering
        \includegraphics[width=\linewidth]{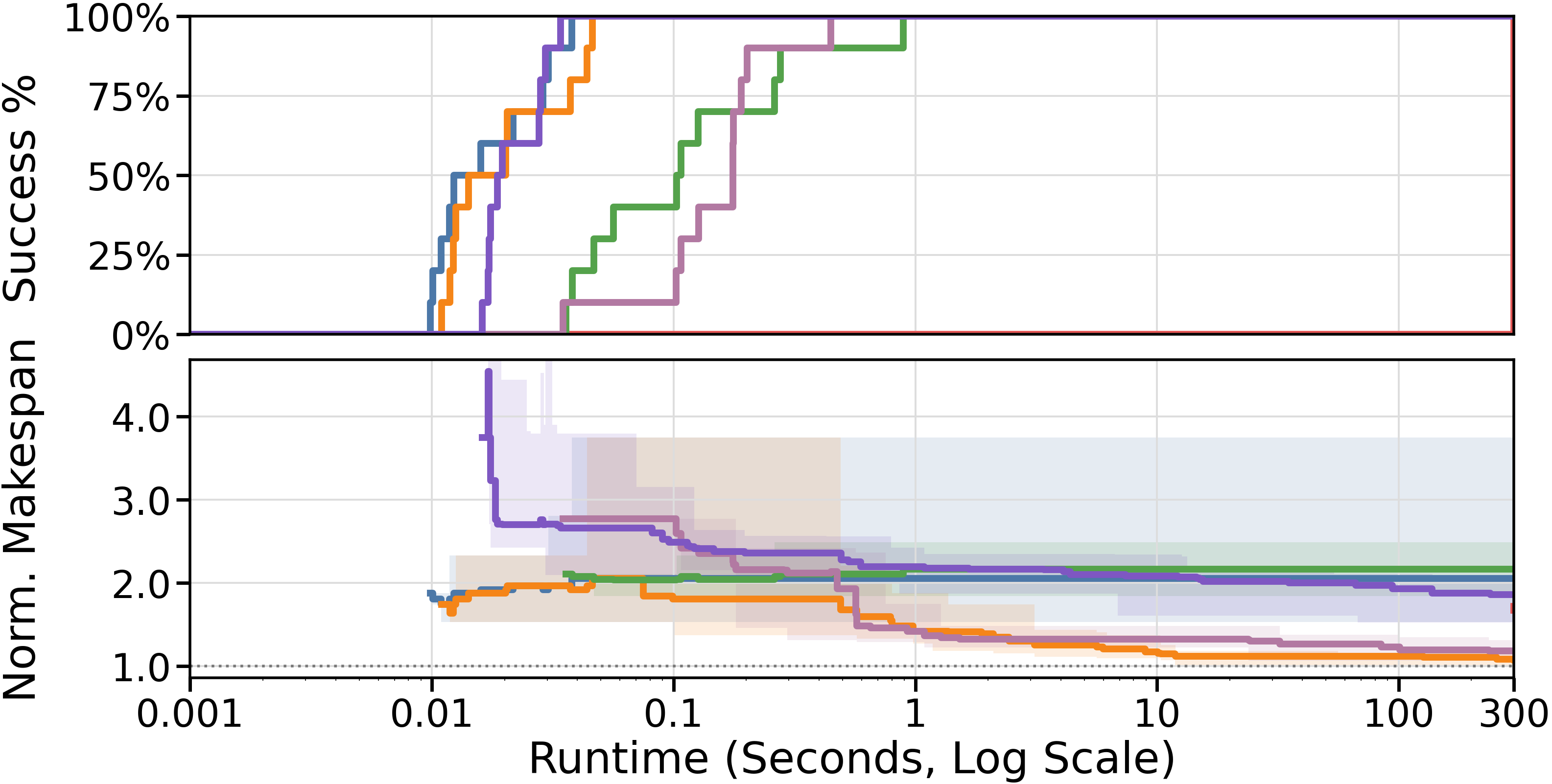}
        \caption{Planar - 4}
        \label{fig:results:planar}
    \end{subfigure}
    \hfill
    \begin{subfigure}[t]{0.45\textwidth}
        \centering
        \includegraphics[width=\linewidth]{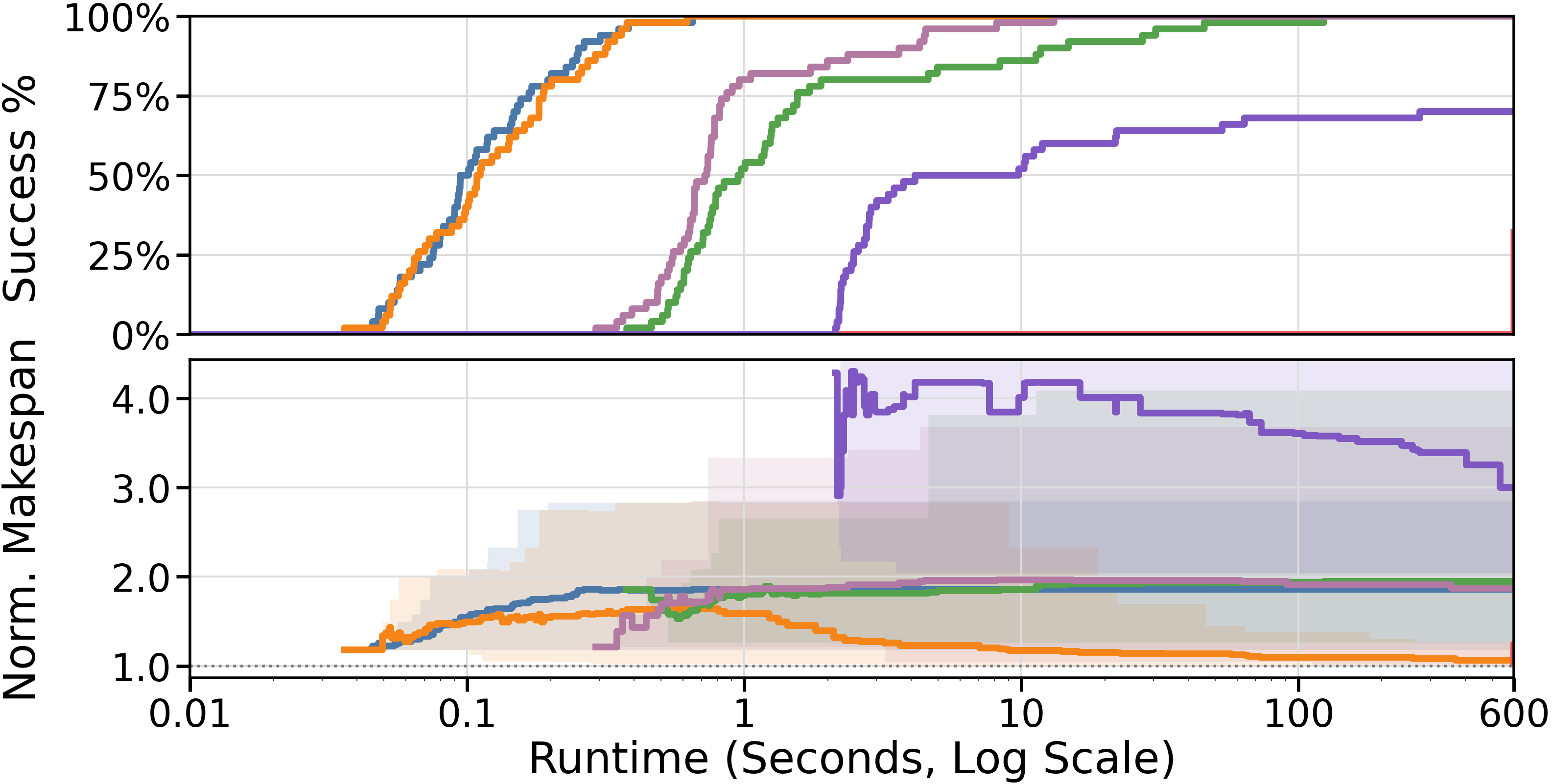}
        \caption{Panda Cage - 4}
        \label{fig:results:panda}
    \end{subfigure}
    \\
    \begin{subfigure}[t]{1.0\textwidth}
        \centering
        \includegraphics[width=\linewidth]{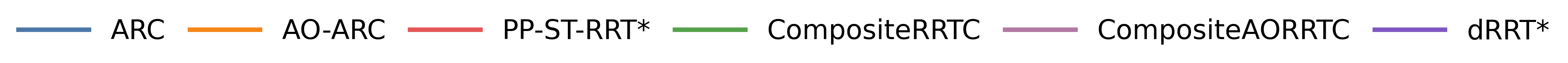}
    \end{subfigure}
    \vspace{-2.0em}
    \caption{
    \small
    (a-c) A representative team size for each 2D scenario.
    (d) Combined statistics for all five tasks for the Panda Cage with four robots.
    (a-d) Cumulative success rate in the top panel. Normalized median makespan (with the range highlighted) in the bottom panel.
    Both are plotted over logarithmic runtime.
    {\ppstrrtstar} values are represented by vertical line at the time limit.
    }
    \label{fig:results}
    \vspace{-2em}
\end{figure*}

We conduct a set of experiments to evaluate time to an initial solution and anytime properties of {\aoarc}.
Section~\ref{sec:experimental-design} details our experimental design.
Section~\ref{sec:results} presents and discusses our results.

\subsection{Experimental Design}
\label{sec:experimental-design}

The experiments are designed to evaluate the performance of {\aoarc} with varying numbers of robots, coordination requirements, and robot complexity.

\subsubsection{Scenarios} 
We include a set of 2D scenarios to control the coordination requirements as the number of robots increases.
These include two 2D mobile robot scenarios with high and low coordination requirements (Figs.~\ref{fig:scenarios-2d-cross},~\ref{fig:scenarios-2d-circle}), and a 2D planar manipulator to increase the robot complexity while forcing coordination (Fig.~\ref{fig:scenarios-planar-cross}).
Additionally, we include a 3D Panda manipulator scenario with 4 and 8 robots operating in a shared workspace.
We test five tasks where each robot has a randomly sampled start and goal with the end effector position inside the cage (Fig.~\ref{fig:panda}).
Each scenario and task was run for 10 trials.
2D scenarios and 3D Panda tasks were run for 300 and 600 seconds, respectively.

\subsubsection{Metrics}
For all scenarios, we measured the time to first solution and recorded the improvement in solution quality over time.
An empirically derived best-known reference makespan $\hat{J}$ is used to normalize the reported makespans.

\subsubsection{Baselines}
We include five baselines.
For anytime behavior, we include a composite space adaptation of {\aorrtc}~\cite{wilson2025aorrtc} (denoted Comp{\aorrtc}) and the multi-robot specific {\drrtstar}~\cite{ssdhb-dsaiammp-20}.
For initial solution time, we include the feasibility planners {\arc} and Comp{\rrtc}.
Additionally, to measure the impact of decoupling on solution quality, we include {\ppstrrtstar}~\cite{kerimov2025si}, the prioritized adaptation of {\strrtstar}~\cite{grothe2022st}.
A random priority order is sampled for each trial.
As prioritized approaches plan for individual robots sequentially, we divide the time limit evenly between each {\strrtstar} call.
This is not a fair comparison for initial solution or anytime behavior because the full solution is not constructed until planning is complete for the last robot 
(addressing this limitation is outside the scope of this work).
It does illustrate the impact of complete decoupling on makespan values for different types of scenarios.

\subsubsection{Hardware}
All experimental trials were run on a Linux workstation (x86\_64, kernel 6.8) equipped with an Intel Core i7-14700F CPU (up to 5.4 GHz) and 32 GB of RAM.

\subsection{Results}
\label{sec:results}

For each scenario, we include plots of the cumulative success rate and median current best makespan across all trials for the largest team size for which all anytime methods obtain $>50\%$ success rate (Fig.~\ref{fig:results}).
The results for the full set of scenarios are reported in Table~\ref{tab:full}.

\begin{table*}[t]
\centering
\tiny
\setlength{\tabcolsep}{2pt}
\renewcommand{\arraystretch}{1.15}
\caption{\small
Combined summary of 2D mobile, planar manipulator, and Panda cage planning results.
Success is the percentage of runs that found a solution by the planning deadline, $T_{\mathrm{init}}$ is the median initial solution time in seconds, $J_{\mathrm{init}}$ and $J_{\mathrm{final}}$ are median makespans normalized by the best-known reference $\hat{J}$, and $J_{\le 1.10}$ is the percentage of runs whose final makespan is within 10\% of $\hat{J}$.
}
\label{tab:combined-planning-summary}
\resizebox{\textwidth}{!}{%
\begin{tabular}{ll|ccccc|ccccc|ccccc}
\hline
 &  & \multicolumn{5}{|c|}{$n=4$} & \multicolumn{5}{c|}{$n=8$} & \multicolumn{5}{c}{$n=16$} \\
\cline{3-7}\cline{8-12}\cline{13-17}
Setting & Method & Success & $T_{\mathrm{init}}$ & $J_{\mathrm{init}}$ & $J_{\mathrm{final}}$ & $J_{\le 1.10}$ & Success & $T_{\mathrm{init}}$ & $J_{\mathrm{init}}$ & $J_{\mathrm{final}}$ & $J_{\le 1.10}$ & Success & $T_{\mathrm{init}}$ & $J_{\mathrm{init}}$ & $J_{\mathrm{final}}$ & $J_{\le 1.10}$ \\
\hline
\multirow{6}{*}{\rotatebox[origin=c]{90}{Cross}} & ARC & \textbf{100\%} & \textbf{0.01} & 1.28 & 1.28 & 0\% & \textbf{100\%} & \textbf{0.03} & \textbf{1.18} & 1.18 & 0\% & \textbf{100\%} & \textbf{0.09} & \textbf{1.11} & 1.11 & 20\% \\
 & \cellcolor{gray!12}AO-ARC & \cellcolor{gray!12}\textbf{100\%} & \cellcolor{gray!12}\textbf{0.01} & \cellcolor{gray!12}1.28 & \cellcolor{gray!12}\textbf{1.00} & \cellcolor{gray!12}\textbf{100\%} & \cellcolor{gray!12}\textbf{100\%} & \cellcolor{gray!12}\textbf{0.03} & \cellcolor{gray!12}\textbf{1.18} & \cellcolor{gray!12}\textbf{1.00} & \cellcolor{gray!12}\textbf{100\%} & \cellcolor{gray!12}\textbf{100\%} & \cellcolor{gray!12}0.10 & \cellcolor{gray!12}\textbf{1.11} & \cellcolor{gray!12}\textbf{1.00} & \cellcolor{gray!12}\textbf{100\%} \\
 & CompRRTC & \textbf{100\%} & 0.04 & 1.17 & 1.17 & 20\% & \textbf{100\%} & 0.14 & 1.33 & 1.33 & 0\% & \textbf{100\%} & 2.40 & 2.02 & 2.02 & 0\% \\
 & \cellcolor{gray!12}CompAORRTC & \cellcolor{gray!12}\textbf{100\%} & \cellcolor{gray!12}0.02 & \cellcolor{gray!12}\textbf{1.15} & \cellcolor{gray!12}\textbf{1.00} & \cellcolor{gray!12}\textbf{100\%} & \cellcolor{gray!12}\textbf{100\%} & \cellcolor{gray!12}0.18 & \cellcolor{gray!12}1.46 & \cellcolor{gray!12}1.01 & \cellcolor{gray!12}\textbf{100\%} & \cellcolor{gray!12}\textbf{100\%} & \cellcolor{gray!12}2.26 & \cellcolor{gray!12}2.41 & \cellcolor{gray!12}1.01 & \cellcolor{gray!12}70\% \\
 & dRRT* & \textbf{100\%} & 0.02 & 1.94 & 1.27 & 0\% & \textbf{100\%} & \textbf{0.03} & 2.12 & 1.85 & 0\% & \textbf{100\%} & 0.12 & 3.84 & 3.02 & 0\% \\
 & \cellcolor{gray!12}PP-ST-RRT* & \cellcolor{gray!12}\textbf{100\%} & \cellcolor{gray!12}$T_{\mathrm{final}}$ & \cellcolor{gray!12}-- & \cellcolor{gray!12}\textbf{1.00} & \cellcolor{gray!12}\textbf{100\%} & \cellcolor{gray!12}\textbf{100\%} & \cellcolor{gray!12}$T_{\mathrm{final}}$ & \cellcolor{gray!12}-- & \cellcolor{gray!12}\textbf{1.00} & \cellcolor{gray!12}\textbf{100\%} & \cellcolor{gray!12}\textbf{100\%} & \cellcolor{gray!12}$T_{\mathrm{final}}$ & \cellcolor{gray!12}-- & \cellcolor{gray!12}\textbf{1.00} & \cellcolor{gray!12}\textbf{100\%} \\
\hline
\multirow{6}{*}{\rotatebox[origin=c]{90}{Circle}} & ARC & \textbf{100\%} & \textbf{0.01} & 1.90 & 1.90 & 0\% & \textbf{100\%} & 0.06 & 3.27 & 3.27 & 0\% & \textbf{100\%} & 3.04 & 5.27 & 5.27 & 0\% \\
 & \cellcolor{gray!12}AO-ARC & \cellcolor{gray!12}\textbf{100\%} & \cellcolor{gray!12}\textbf{0.01} & \cellcolor{gray!12}1.90 & \cellcolor{gray!12}\textbf{1.00} & \cellcolor{gray!12}\textbf{100\%} & \cellcolor{gray!12}\textbf{100\%} & \cellcolor{gray!12}0.07 & \cellcolor{gray!12}3.27 & \cellcolor{gray!12}1.13 & \cellcolor{gray!12}40\% & \cellcolor{gray!12}\textbf{100\%} & \cellcolor{gray!12}\textbf{2.65} & \cellcolor{gray!12}5.27 & \cellcolor{gray!12}5.27 & \cellcolor{gray!12}0\% \\
 & CompRRTC & \textbf{100\%} & 0.02 & \textbf{1.29} & 1.29 & 0\% & \textbf{100\%} & 0.18 & \textbf{1.84} & 1.84 & 0\% & \textbf{100\%} & 64.88 & 4.06 & 4.06 & 0\% \\
 & \cellcolor{gray!12}CompAORRTC & \cellcolor{gray!12}\textbf{100\%} & \cellcolor{gray!12}0.02 & \cellcolor{gray!12}1.36 & \cellcolor{gray!12}\textbf{1.00} & \cellcolor{gray!12}\textbf{100\%} & \cellcolor{gray!12}\textbf{100\%} & \cellcolor{gray!12}0.18 & \cellcolor{gray!12}2.30 & \cellcolor{gray!12}1.14 & \cellcolor{gray!12}0\% & \cellcolor{gray!12}70\% & \cellcolor{gray!12}135.04 & \cellcolor{gray!12}\textbf{3.91} & \cellcolor{gray!12}3.91 & \cellcolor{gray!12}0\% \\
 & dRRT* & \textbf{100\%} & 0.02 & 1.94 & 1.36 & 0\% & \textbf{100\%} & \textbf{0.05} & 3.87 & 2.23 & 0\% & 50\% & 68.70 & 6.39 & 6.39 & 0\% \\
 & \cellcolor{gray!12}PP-ST-RRT* & \cellcolor{gray!12}\textbf{100\%} & \cellcolor{gray!12}$T_{\mathrm{final}}$ & \cellcolor{gray!12}-- & \cellcolor{gray!12}1.04 & \cellcolor{gray!12}\textbf{100\%} & \cellcolor{gray!12}\textbf{100\%} & \cellcolor{gray!12}$T_{\mathrm{final}}$ & \cellcolor{gray!12}-- & \cellcolor{gray!12}\textbf{1.02} & \cellcolor{gray!12}\textbf{100\%} & \cellcolor{gray!12}\textbf{100\%} & \cellcolor{gray!12}$T_{\mathrm{final}}$ & \cellcolor{gray!12}-- & \cellcolor{gray!12}\textbf{1.02} & \cellcolor{gray!12}\textbf{100\%} \\
\hline
\multirow{6}{*}{\rotatebox[origin=c]{90}{Planar}} & ARC & \textbf{100\%} & \textbf{0.01} & \textbf{2.05} & 2.05 & 0\% & \textbf{100\%} & \textbf{0.05} & 2.33 & 2.33 & 0\% & \textbf{100\%} & \textbf{0.12} & \textbf{2.39} & 2.39 & 0\% \\
 & \cellcolor{gray!12}AO-ARC & \cellcolor{gray!12}\textbf{100\%} & \cellcolor{gray!12}0.02 & \cellcolor{gray!12}\textbf{2.05} & \cellcolor{gray!12}\textbf{1.07} & \cellcolor{gray!12}\textbf{70\%} & \cellcolor{gray!12}\textbf{100\%} & \cellcolor{gray!12}0.08 & \cellcolor{gray!12}2.33 & \cellcolor{gray!12}\textbf{1.15} & \cellcolor{gray!12}\textbf{40\%} & \cellcolor{gray!12}\textbf{100\%} & \cellcolor{gray!12}0.16 & \cellcolor{gray!12}\textbf{2.39} & \cellcolor{gray!12}\textbf{1.05} & \cellcolor{gray!12}\textbf{90\%} \\
 & CompRRTC & \textbf{100\%} & 0.11 & 2.17 & 2.17 & 0\% & 40\% & 31.24 & \textbf{2.20} & 2.20 & 0\% & 0\% & -- & -- & -- & 0\% \\
 & \cellcolor{gray!12}CompAORRTC & \cellcolor{gray!12}\textbf{100\%} & \cellcolor{gray!12}0.18 & \cellcolor{gray!12}2.20 & \cellcolor{gray!12}1.18 & \cellcolor{gray!12}20\% & \cellcolor{gray!12}70\% & \cellcolor{gray!12}81.26 & \cellcolor{gray!12}2.22 & \cellcolor{gray!12}2.22 & \cellcolor{gray!12}0\% & \cellcolor{gray!12}0\% & \cellcolor{gray!12}-- & \cellcolor{gray!12}-- & \cellcolor{gray!12}-- & \cellcolor{gray!12}0\% \\
 & dRRT* & \textbf{100\%} & 0.02 & 3.03 & 1.86 & 0\% & \textbf{100\%} & 0.28 & 4.53 & 3.05 & 0\% & 10\% & 4.08 & 14.78 & 14.78 & 0\% \\
 & \cellcolor{gray!12}PP-ST-RRT* & \cellcolor{gray!12}\textbf{100\%} & \cellcolor{gray!12}$T_{\mathrm{final}}$ & \cellcolor{gray!12}-- & \cellcolor{gray!12}1.67 & \cellcolor{gray!12}0\% & \cellcolor{gray!12}\textbf{100\%} & \cellcolor{gray!12}$T_{\mathrm{final}}$ & \cellcolor{gray!12}-- & \cellcolor{gray!12}1.97 & \cellcolor{gray!12}0\% & \cellcolor{gray!12}\textbf{100\%} & \cellcolor{gray!12}$T_{\mathrm{final}}$ & \cellcolor{gray!12}-- & \cellcolor{gray!12}2.38 & \cellcolor{gray!12}0\% \\
\hline
\multirow{6}{*}{\rotatebox[origin=c]{90}{Panda Task 1}} & ARC & \textbf{100\%} & 0.07 & \textbf{1.66} & 1.66 & 0\% & \textbf{100\%} & \textbf{0.57} & \textbf{1.92} & 1.92 & 0\% & -- & -- & -- & -- & -- \\
 & \cellcolor{gray!12}AO-ARC & \cellcolor{gray!12}\textbf{100\%} & \cellcolor{gray!12}\textbf{0.06} & \cellcolor{gray!12}\textbf{1.66} & \cellcolor{gray!12}\textbf{1.02} & \cellcolor{gray!12}\textbf{80\%} & \cellcolor{gray!12}\textbf{100\%} & \cellcolor{gray!12}0.67 & \cellcolor{gray!12}\textbf{1.92} & \cellcolor{gray!12}\textbf{1.35} & \cellcolor{gray!12}\textbf{30\%} & \cellcolor{gray!12}-- & \cellcolor{gray!12}-- & \cellcolor{gray!12}-- & \cellcolor{gray!12}-- & \cellcolor{gray!12}-- \\
 & CompRRTC & \textbf{100\%} & 0.75 & 2.02 & 2.02 & 0\% & 0\% & -- & -- & -- & 0\% & -- & -- & -- & -- & -- \\
 & \cellcolor{gray!12}CompAORRTC & \cellcolor{gray!12}\textbf{100\%} & \cellcolor{gray!12}0.63 & \cellcolor{gray!12}2.06 & \cellcolor{gray!12}2.06 & \cellcolor{gray!12}0\% & \cellcolor{gray!12}0\% & \cellcolor{gray!12}-- & \cellcolor{gray!12}-- & \cellcolor{gray!12}-- & \cellcolor{gray!12}0\% & \cellcolor{gray!12}-- & \cellcolor{gray!12}-- & \cellcolor{gray!12}-- & \cellcolor{gray!12}-- & \cellcolor{gray!12}-- \\
 & dRRT* & 90\% & 2.58 & 4.75 & 3.39 & 0\% & 0\% & -- & -- & -- & 0\% & -- & -- & -- & -- & -- \\
 & \cellcolor{gray!12}PP-ST-RRT* & \cellcolor{gray!12}0\% & \cellcolor{gray!12}$T_{\mathrm{final}}$ & \cellcolor{gray!12}-- & \cellcolor{gray!12}-- & \cellcolor{gray!12}0\% & \cellcolor{gray!12}0\% & \cellcolor{gray!12}$T_{\mathrm{final}}$ & \cellcolor{gray!12}-- & \cellcolor{gray!12}-- & \cellcolor{gray!12}0\% & \cellcolor{gray!12}-- & \cellcolor{gray!12}-- & \cellcolor{gray!12}-- & \cellcolor{gray!12}-- & \cellcolor{gray!12}-- \\
\hline
\multirow{6}{*}{\rotatebox[origin=c]{90}{Panda Task 2}} & ARC & \textbf{100\%} & \textbf{0.09} & \textbf{1.64} & 1.64 & 0\% & \textbf{100\%} & \textbf{0.47} & \textbf{2.39} & 2.39 & 0\% & -- & -- & -- & -- & -- \\
 & \cellcolor{gray!12}AO-ARC & \cellcolor{gray!12}\textbf{100\%} & \cellcolor{gray!12}0.11 & \cellcolor{gray!12}\textbf{1.64} & \cellcolor{gray!12}\textbf{1.10} & \cellcolor{gray!12}\textbf{50\%} & \cellcolor{gray!12}\textbf{100\%} & \cellcolor{gray!12}0.50 & \cellcolor{gray!12}\textbf{2.39} & \cellcolor{gray!12}\textbf{1.02} & \cellcolor{gray!12}\textbf{90\%} & \cellcolor{gray!12}-- & \cellcolor{gray!12}-- & \cellcolor{gray!12}-- & \cellcolor{gray!12}-- & \cellcolor{gray!12}-- \\
 & CompRRTC & \textbf{100\%} & 1.29 & 1.70 & 1.70 & 0\% & 50\% & 491.41 & 4.68 & 4.68 & 0\% & -- & -- & -- & -- & -- \\
 & \cellcolor{gray!12}CompAORRTC & \cellcolor{gray!12}\textbf{100\%} & \cellcolor{gray!12}0.65 & \cellcolor{gray!12}1.73 & \cellcolor{gray!12}1.70 & \cellcolor{gray!12}0\% & \cellcolor{gray!12}30\% & \cellcolor{gray!12}514.76 & \cellcolor{gray!12}3.06 & \cellcolor{gray!12}3.06 & \cellcolor{gray!12}0\% & \cellcolor{gray!12}-- & \cellcolor{gray!12}-- & \cellcolor{gray!12}-- & \cellcolor{gray!12}-- & \cellcolor{gray!12}-- \\
 & dRRT* & 80\% & 11.50 & 3.42 & 2.93 & 0\% & 0\% & -- & -- & -- & 0\% & -- & -- & -- & -- & -- \\
 & \cellcolor{gray!12}PP-ST-RRT* & \cellcolor{gray!12}30\% & \cellcolor{gray!12}$T_{\mathrm{final}}$ & \cellcolor{gray!12}-- & \cellcolor{gray!12}\textbf{1.10} & \cellcolor{gray!12}10\% & \cellcolor{gray!12}30\% & \cellcolor{gray!12}$T_{\mathrm{final}}$ & \cellcolor{gray!12}-- & \cellcolor{gray!12}1.56 & \cellcolor{gray!12}0\% & \cellcolor{gray!12}-- & \cellcolor{gray!12}-- & \cellcolor{gray!12}-- & \cellcolor{gray!12}-- & \cellcolor{gray!12}-- \\
\hline
\multirow{6}{*}{\rotatebox[origin=c]{90}{Panda Task 3}} & ARC & \textbf{100\%} & \textbf{0.06} & \textbf{1.41} & 1.41 & 0\% & \textbf{40\%} & 2.08 & \textbf{1.67} & 1.67 & 0\% & -- & -- & -- & -- & -- \\
 & \cellcolor{gray!12}AO-ARC & \cellcolor{gray!12}\textbf{100\%} & \cellcolor{gray!12}0.07 & \cellcolor{gray!12}\textbf{1.41} & \cellcolor{gray!12}\textbf{1.01} & \cellcolor{gray!12}\textbf{100\%} & \cellcolor{gray!12}\textbf{40\%} & \cellcolor{gray!12}\textbf{1.79} & \cellcolor{gray!12}\textbf{1.67} & \cellcolor{gray!12}\textbf{1.23} & \cellcolor{gray!12}\textbf{20\%} & \cellcolor{gray!12}-- & \cellcolor{gray!12}-- & \cellcolor{gray!12}-- & \cellcolor{gray!12}-- & \cellcolor{gray!12}-- \\
 & CompRRTC & \textbf{100\%} & 0.57 & 1.56 & 1.56 & 0\% & 0\% & -- & -- & -- & 0\% & -- & -- & -- & -- & -- \\
 & \cellcolor{gray!12}CompAORRTC & \cellcolor{gray!12}\textbf{100\%} & \cellcolor{gray!12}0.64 & \cellcolor{gray!12}1.56 & \cellcolor{gray!12}1.04 & \cellcolor{gray!12}\textbf{100\%} & \cellcolor{gray!12}0\% & \cellcolor{gray!12}-- & \cellcolor{gray!12}-- & \cellcolor{gray!12}-- & \cellcolor{gray!12}0\% & \cellcolor{gray!12}-- & \cellcolor{gray!12}-- & \cellcolor{gray!12}-- & \cellcolor{gray!12}-- & \cellcolor{gray!12}-- \\
 & dRRT* & \textbf{100\%} & 2.40 & 5.01 & 2.62 & 0\% & 0\% & -- & -- & -- & 0\% & -- & -- & -- & -- & -- \\
 & \cellcolor{gray!12}PP-ST-RRT* & \cellcolor{gray!12}70\% & \cellcolor{gray!12}$T_{\mathrm{final}}$ & \cellcolor{gray!12}-- & \cellcolor{gray!12}1.04 & \cellcolor{gray!12}70\% & \cellcolor{gray!12}0\% & \cellcolor{gray!12}$T_{\mathrm{final}}$ & \cellcolor{gray!12}-- & \cellcolor{gray!12}-- & \cellcolor{gray!12}0\% & \cellcolor{gray!12}-- & \cellcolor{gray!12}-- & \cellcolor{gray!12}-- & \cellcolor{gray!12}-- & \cellcolor{gray!12}-- \\
\hline
\multirow{6}{*}{\rotatebox[origin=c]{90}{Panda Task 4}} & ARC & \textbf{100\%} & \textbf{0.18} & \textbf{2.26} & 2.26 & 0\% & \textbf{90\%} & 1.50 & \textbf{1.94} & 1.94 & 0\% & -- & -- & -- & -- & -- \\
 & \cellcolor{gray!12}AO-ARC & \cellcolor{gray!12}\textbf{100\%} & \cellcolor{gray!12}0.28 & \cellcolor{gray!12}\textbf{2.26} & \cellcolor{gray!12}\textbf{1.14} & \cellcolor{gray!12}\textbf{20\%} & \cellcolor{gray!12}\textbf{90\%} & \cellcolor{gray!12}\textbf{1.34} & \cellcolor{gray!12}\textbf{1.94} & \cellcolor{gray!12}\textbf{1.50} & \cellcolor{gray!12}\textbf{10\%} & \cellcolor{gray!12}-- & \cellcolor{gray!12}-- & \cellcolor{gray!12}-- & \cellcolor{gray!12}-- & \cellcolor{gray!12}-- \\
 & CompRRTC & \textbf{100\%} & 13.25 & 3.20 & 3.20 & 0\% & 0\% & -- & -- & -- & 0\% & -- & -- & -- & -- & -- \\
 & \cellcolor{gray!12}CompAORRTC & \cellcolor{gray!12}\textbf{100\%} & \cellcolor{gray!12}3.96 & \cellcolor{gray!12}2.65 & \cellcolor{gray!12}2.65 & \cellcolor{gray!12}0\% & \cellcolor{gray!12}0\% & \cellcolor{gray!12}-- & \cellcolor{gray!12}-- & \cellcolor{gray!12}-- & \cellcolor{gray!12}0\% & \cellcolor{gray!12}-- & \cellcolor{gray!12}-- & \cellcolor{gray!12}-- & \cellcolor{gray!12}-- & \cellcolor{gray!12}-- \\
 & dRRT* & 20\% & 2.65 & 7.47 & 4.77 & 0\% & 0\% & -- & -- & -- & 0\% & -- & -- & -- & -- & -- \\
 & \cellcolor{gray!12}PP-ST-RRT* & \cellcolor{gray!12}50\% & \cellcolor{gray!12}$T_{\mathrm{final}}$ & \cellcolor{gray!12}-- & \cellcolor{gray!12}1.31 & \cellcolor{gray!12}0\% & \cellcolor{gray!12}0\% & \cellcolor{gray!12}$T_{\mathrm{final}}$ & \cellcolor{gray!12}-- & \cellcolor{gray!12}-- & \cellcolor{gray!12}0\% & \cellcolor{gray!12}-- & \cellcolor{gray!12}-- & \cellcolor{gray!12}-- & \cellcolor{gray!12}-- & \cellcolor{gray!12}-- \\
\hline
\multirow{6}{*}{\rotatebox[origin=c]{90}{Panda Task 5}} & ARC & \textbf{100\%} & 0.21 & \textbf{1.83} & 1.83 & 0\% & \textbf{100\%} & \textbf{1.90} & \textbf{2.53} & 2.53 & 0\% & -- & -- & -- & -- & -- \\
 & \cellcolor{gray!12}AO-ARC & \cellcolor{gray!12}\textbf{100\%} & \cellcolor{gray!12}\textbf{0.20} & \cellcolor{gray!12}\textbf{1.83} & \cellcolor{gray!12}\textbf{1.14} & \cellcolor{gray!12}\textbf{40\%} & \cellcolor{gray!12}\textbf{100\%} & \cellcolor{gray!12}1.91 & \cellcolor{gray!12}\textbf{2.53} & \cellcolor{gray!12}2.26 & \cellcolor{gray!12}0\% & \cellcolor{gray!12}-- & \cellcolor{gray!12}-- & \cellcolor{gray!12}-- & \cellcolor{gray!12}-- & \cellcolor{gray!12}-- \\
 & CompRRTC & \textbf{100\%} & 1.08 & 2.00 & 2.00 & 0\% & 0\% & -- & -- & -- & 0\% & -- & -- & -- & -- & -- \\
 & \cellcolor{gray!12}CompAORRTC & \cellcolor{gray!12}\textbf{100\%} & \cellcolor{gray!12}0.70 & \cellcolor{gray!12}1.98 & \cellcolor{gray!12}1.98 & \cellcolor{gray!12}0\% & \cellcolor{gray!12}0\% & \cellcolor{gray!12}-- & \cellcolor{gray!12}-- & \cellcolor{gray!12}-- & \cellcolor{gray!12}0\% & \cellcolor{gray!12}-- & \cellcolor{gray!12}-- & \cellcolor{gray!12}-- & \cellcolor{gray!12}-- & \cellcolor{gray!12}-- \\
 & dRRT* & 60\% & 10.29 & 4.51 & 3.94 & 0\% & 0\% & -- & -- & -- & 0\% & -- & -- & -- & -- & -- \\
 & \cellcolor{gray!12}PP-ST-RRT* & \cellcolor{gray!12}10\% & \cellcolor{gray!12}$T_{\mathrm{final}}$ & \cellcolor{gray!12}-- & \cellcolor{gray!12}1.33 & \cellcolor{gray!12}0\% & \cellcolor{gray!12}10\% & \cellcolor{gray!12}$T_{\mathrm{final}}$ & \cellcolor{gray!12}-- & \cellcolor{gray!12}\textbf{1.00} & \cellcolor{gray!12}\textbf{10\%} & \cellcolor{gray!12}-- & \cellcolor{gray!12}-- & \cellcolor{gray!12}-- & \cellcolor{gray!12}-- & \cellcolor{gray!12}-- \\
\hline
\end{tabular}
}
\label{tab:full}
\end{table*}

\subsubsection{2D Scenarios}
In all 2D experiments, {\arc} and {\aoarc} were first and second in initial solution time $T_{init}$, often by a considerable margin as the number of robots increased.
This is consistent with prior {\arc} results and {\aoarc} maintained this with a marginal overhead.

In 8/9 of the 2D experiments, {\aoarc} achieved a median final makespan $J_{\mathrm{final}}$ and percentage of trials with $J_{\mathrm{final}}\le 1.10\hat{J}$ (denoted $J_{\le 1.10}$ in Table~\ref{tab:full}) equal to or better than the two other asymptotically optimal methods, Comp{\aorrtc} and {\drrtstar}.
The one exception is the 16 robot circle scenario, where no method improved its median makespan from the initial solution and {\aoarc} solved 100\% of the trials while Comp{\aorrtc} and {\drrtstar} solved 70\% and 50\%.
The prioritized multi-robot variant of {\strrtstar} was able to produce high quality solutions after 300 seconds in the 2D mobile robot scenarios where decoupled reasoning was sufficient, but fell behind {\aoarc} in final solution quality as the robot complexity increased.

\subsubsection{3D Panda Manipulator Teams}
In the 3D Panda cage scenario, we again see the faster initial solution times of {\arc} and {\aoarc} and the significantly higher success rate as the team size increases.
Across all five tasks for both 4 and 8 robot scenarios, {\aoarc} produces a better median $J_{\mathrm{final}}$ than Comp{\aorrtc} and {\drrtstar} and only loses to {\ppstrrtstar} on Task 5 for 8 robots, where {\ppstrrtstar} solves only one of the ten trials compared to the 100\% success rate of {\aoarc}.
With a makespan cost metric, this likely implies that the optimal solution to task 5 is dominated by one long robot trajectory and the one successful trial of {\ppstrrtstar} had it early in the randomized robot priority list.
This trend of high quality solutions with low success rate is consistent across all of the Panda tasks for {\ppstrrtstar}.

{\arc} and {\aoarc} also show more robustness across the varying task difficulties.
Comp{\aorrtc} did better than {\drrtstar} but still failed to solve any but the easiest 8 panda task (as determined by the reported {\aoarc} $T_{init}$ and $J_{\mathrm{final}}$ values).
In that task, the {\arc} variants demonstrated a 1000X speedup in initial solution time, at least double the success rate, and a 3X reduction in normalized solution cost.

\section{Conclusion}
\label{sec:conclusion}

We present {\aoarc}, a fast, anytime {\mrmp} approach matching state-of-the-art feasibility solvers in initial planning time and scalability while providing asymptotic optimality and fast solution convergence.
{\aoarc} achieves this by combining the {\aox} meta-algorithm with a makespan bound that remains constant across the robot decompositions used by {\arc}.
Future work will extend these techniques to other standard {\mrmp} cost metrics such as sum-of-costs and more complex state spaces like~\cite{qin2025k} and~\cite{hartmann2025sampling}.

\bibliographystyle{IEEEtran}
\bibliography{robotics}

\end{document}